\documentclass[lettersize,journal]{IEEEtran}
\usepackage[compact]{titlesec}
\usepackage{amsmath,amsfonts}
\usepackage{algorithmic}
\usepackage{algorithm}
\usepackage{array}
\usepackage{subfig}
\usepackage{xcolor}
\usepackage{textcomp}
\usepackage{stfloats}
\usepackage{url}
\usepackage{verbatim}
\usepackage{graphicx}
\usepackage{cite}
\usepackage{multirow}
\usepackage[compact]{titlesec}

\usepackage{amsthm}
\newtheorem{definition}{Definition}
\newtheorem{theorem}{Theorem}

\ifodd 1

\else

\fi

\begin{document}

\title{State-Aware Perturbation Optimization for Robust Deep Reinforcement Learning}


\author{
    Zongyuan Zhang,
    Tianyang Duan,
    Zheng Lin,
    Dong Huang, 
    Zihan Fang,
    Zekai Sun, 
    Ling Xiong,~\IEEEmembership{Member,~IEEE,}
    Hongbin Liang,~\IEEEmembership{Member,~IEEE,}
    Heming Cui,~\IEEEmembership{Member,~IEEE,} 
    and Yong Cui,~\IEEEmembership{Member,~IEEE} 
    \thanks{Z. Zhang, T. Duan, D. Huang, Z. Sun, and H. Cui are with the Department of Computer Science, The University of Hong Kong, Hong Kong SAR, China (e-mail: zyzhang2@cs.hku.hk; tyduan@cs.hku.hk; dhuang@cs.hku.hk; zksun@cs.hku.hk; heming@cs.hku.hk).}
    \thanks{Z. Lin is with the Department of Electrical and Electronic Engineering, The University of Hong Kong, Hong Kong SAR, China (e-mail: linzheng@eee.hku.hk).}
    \thanks{Z. Fang is with the Department of Computer Science, City University of Hong Kong, Hong Kong SAR, China (e-mail: zihanfang3-c@my.cityu.edu.hk).}
    \thanks{L. Xiong is with the School of Computer and Software Engineering, Xihua University, Chengdu 610039, China. (e-mail: lingdonghua99@163.com)}
    \thanks{H. Liang is with the National United Engineering Laboratory of Integrated and Intelligent Transportation, and the National Engineering Laboratory of Integrated Transportation Big Data Application Technology, Southwest Jiaotong University, Chengdu 611756, China (e-mail: hbliang@swjtu.edu.cn).}
    \thanks{Y. Cui is with the Department of Computer Science and Technology, Tsinghua University, Beijing 100084, China (e-mail: cuiyong@tsinghua.edu.cn).}
    \thanks{(Corresponding author: Tianyang Duan; Zheng Lin)}
}



\maketitle

\begin{abstract}
Recently, deep reinforcement learning (DRL) has emerged as a promising approach for robotic control. However, the deployment of DRL in real-world robots is hindered by its sensitivity to environmental perturbations. While existing white-box adversarial attacks rely on local gradient information and apply uniform perturbations across all states to evaluate DRL robustness, they fail to account for temporal dynamics and state-specific vulnerabilities. To combat the above challenge, we first conduct a theoretical analysis of white-box attacks in DRL by establishing 
the adversarial victim-dynamics Markov decision process (AVD-MDP), to derive the necessary and sufficient conditions for a successful attack. Based on this, we propose a selective state-aware reinforcement adversarial attack method, named STAR, to optimize perturbation stealthiness and state visitation dispersion. STAR first employs a soft mask-based state-targeting mechanism to minimize redundant perturbations, enhancing stealthiness and attack effectiveness. Then, it incorporates an information-theoretic optimization objective to maximize mutual information between perturbations, environmental states, and victim actions, ensuring a dispersed state-visitation distribution that steers the victim agent into vulnerable states for maximum return reduction. Extensive experiments demonstrate that STAR outperforms state-of-the-art benchmarks.

\end{abstract}


\section{Introduction}
Robotic systems have become increasingly prevalent in mobile and distributed applications, ranging from autonomous navigation \cite{wei2024autonomous,lin2022channel,hu2024toward, 10135137,lin2024adaptsfl,hu2024agentscodriver,song2023emma}, intelligent transportation \cite{10323097,lin2022tracking, khalil2024advanced,lin2023pushing,fang2024ic3m,lin2022v2i}, and industrial manufacturing \cite{han2023survey, 10100908}. While traditional robotic control methods have achieved considerable success in structured environments with predefined tasks, they struggle to adapt to dynamic scenarios, handle uncertainties, or learn from experience.  Deep reinforcement learning (DRL) has emerged as a promising alternative for robotic control \cite{lin2021softgym,duan2025rethinking}. Unlike traditional approaches that rely on manually designed rules, DRL enables agents to acquire optimal behaviors through trial-and-error interactions with their environment. By learning a policy that maps environmental states to actions while maximizing long-term rewards, DRL is particularly effective for complex robotic tasks involving delayed feedback and temporal dependencies.

\begin{figure}
    \centering
    \includegraphics[width=0.83\linewidth]{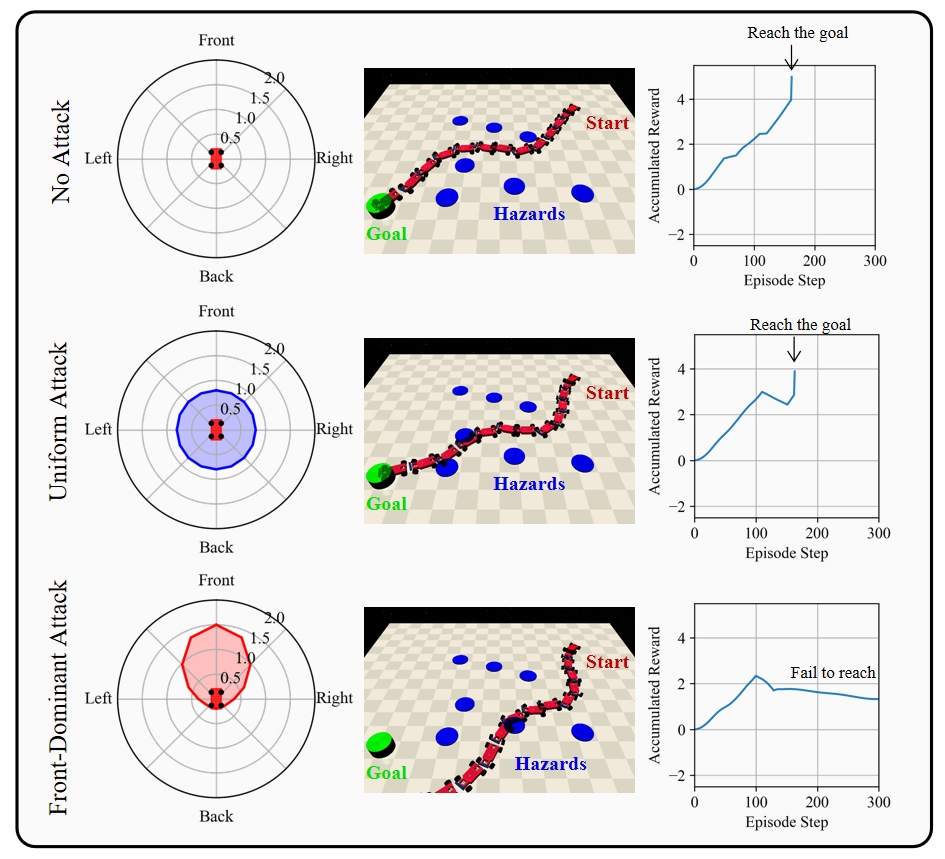}
    \caption{Performance comparison of three attack strategies in a four-wheeled robot navigation task: No Attack (top), Uniform Attack (middle), and Front-Dominant Attack (bottom). Each row presents the attack intensity distribution (left), navigation trajectory (middle), and accumulated reward (right). The Front-Dominant Attack exhibits the highest efficacy in disrupting navigation by concentrating perturbations in the frontal direction.}
    \label{fig:intro}
\end{figure}

The robustness of DRL policies is crucial for their successful deployment in real-world robotic applications, as DRL agents can be highly sensitive to environmental perturbations\cite{shi2024distributionally,yuan2024itpatch}. Minor variations in the input state - whether arising from sensor noise, environmental changes, or adversarial attacks - can significantly impact an agent's decision-making process and potentially trigger catastrophic failures. White-box adversarial attacks serve as an effective tool for evaluating RL robustness by exposing potential vulnerabilities in the learned policies. With full access to the model architecture and parameters, white-box attacks can systematically analyze policy networks to generate perturbations\cite{schott2024robust}. Through adversarial training with perturbed samples generated by such attacks, DRL agents can learn to maintain robust performance under state perturbations, enhancing reliability in real-world deployments.



However, existing white-box attack methods face significant challenges in targeting deep reinforcement learning (DRL) agents, as they primarily rely on local gradient information to generate perturbations \cite{duan2021advdrop, chen2024diffusion, chen2024content}. Adapted from supervised learning attack paradigms, these methods assume temporal independence and focus on instantaneous state-action mappings, neglecting the temporal dynamics inherent in Markov Decision Processes (MDPs). As a result, they fail to generate perturbations that can effectively disrupt the agent's cumulative rewards over extended time horizons. 
Furthermore, these methods apply perturbations indiscriminately across all states without identifying key features that critically impact performance. This is particularly problematic in high-dimensional state spaces, where only a subset of variables—such as specific joint angles in robotic manipulation—are crucial for policy execution. Though some works attempt to weight perturbations based on policy network gradient magnitudes \cite{oikarinen2021robust, hickling2023robust, haydari2021adversarial}, aiming to maximize behavioral deviation from the original policy. Nevertheless, in DRL, this does not directly align with the core attack objective of minimizing cumulative rewards, as agents can adapt by selecting alternative actions to maintain comparable long-term performance.

To investigate the limitations of state-agnostic attacks, we conduct experiments using a four-wheeled robot navigation task in the Safety Gymnasium environment \cite{ji2023safety}. As shown in Figure \ref{fig:intro}, the robot's objective is to navigate from the start position to the goal while avoiding hazards. We evaluate three attack strategies: (i) \textit{No Attack}, where the robot operates without interference; (ii) \textit{Uniform Attack}, where perturbations are applied uniformly across all state dimensions; and (iii) \textit{Front-Dominant Attack}, where perturbations are concentrated in the robot's frontal direction. For each strategy, we analyzed the perturbation distribution across state dimensions, the resulting trajectory, and the accumulated reward over time. Under \textit{No Attack}, the robot follows an optimal trajectory, successfully avoiding hazards and reaching the goal, with the accumulated reward increasing steadily. In contrast, \textit{Uniform Attack} introduces perturbations across all state dimensions, causing minor trajectory deviations. While the robot occasionally approaches hazards, it still reaches the goal, with a slight dip in the accumulated reward curve before near-optimal performance is regained. Finally, \textit{Front-Dominant Attack} applies perturbations predominantly in the frontal direction, severely disrupting the robot’s navigation. This results in significant trajectory deviation, with the robot encountering hazards and failing to reach the goal, leading to a substantial reduction in the accumulated reward. These experiments highlight the importance of attack selectivity. Using the same perturbation budget as \textit{Uniform Attack}, \textit{Front-Dominant Attack} proves more effective by targeting specific state dimensions, exposing the vulnerability of DRL agents to state-specific perturbations.

To address the above issue, the paper aims to address a fundamental research question: \emph{How can we develop a state-aware adversarial attack framework that identifies and exploits vulnerable states in DRL policies while accounting for long-term reward impact?} To this end, we first propose the Adversarial Victim-Dynamics Markov Decision Process (AVD-MDP) to formalize white-box adversarial attacks in DRL, enabling a systematic analysis of adversarial-victim interactions under perturbation constraints. Based on this, we derive necessary and sufficient conditions for successful attacks and identify two key properties essential for effective adversarial strategies: stealthiness of perturbations and dispersion in state visitation. Guided by these insights, we then propose \textbf{S}elective S\textbf{T}ate-\textbf{A}ware \textbf{R}einforcement adversarial attack (STAR), a white-box attack algorithm that jointly optimizes these properties. STAR employs a soft mask-based state-targeting mechanism to minimize perturbations on redundant state dimensions, enhancing stealthiness while maintaining attack efficacy. Additionally, it incorporates an information-theoretic objective to maximize mutual information between adversarial perturbations, environmental states, and victim actions, promoting a dispersed state-visitation pattern. By strategically inducing the victim agent to visit vulnerable states, STAR maximizes return reduction under a fixed perturbation budget. key contribution of the paper can be summarized as follows:
\begin{itemize}
\item We propose a novel adversarial attack framework for DRL that systematically accounts for temporal dependencies and state-specific vulnerabilities, addressing key limitations in existing white-box attack methods.
\item  We theoretically analyze the interaction between adversarial and victim policies by establishing the adversarial victim-dynamics Markov decision process (AVD-MDP), to derive the necessary and sufficient conditions for a successful attack.
\item Based on the derived interaction, we design STAR, a principled white-box attack algorithm that optimizes both stealthiness and distribution via a soft mask-based state-targeting mechanism and an information-theoretic optimization objective to maximize attack effectiveness under a certain perturbation budget.
\item We empirically evaluate STAR with extensive experiments. The results demonstrate that STAR outperforms state-of-the-art frameworks in evaluating the robustness DRL.
\end{itemize}

The rest of the paper is organized as follows. Section II discusses related work and technical limitations. Section III elaborates
on the system model and Section IV presents theoretical analysis of the AVD-MDP process. Section V presents the system design of STAR. Section VI details the experimental setup, followed by performance evaluation in Section VII. Finally, conclusions are presented in Section VIII.

\section{Related Work}

Adversarial attacks expose the vulnerabilities of deep neural networks (DNNs) by introducing carefully crafted perturbations into input data, leading to incorrect predictions during inference~\cite{ZW_TCOM_2024,wang2024ultralola}. These perturbations, though imperceptible to humans, can significantly alter DNN outputs \cite{akhtar2021advances}. White-box attacks constitute a critical category of adversarial attacks, where the adversary has full access to the model, including its architecture, parameters, and gradients. Fast Gradient Sign Method (FGSM) \cite{goodfellow2014explaining} is a seminal white-box attack that efficiently generates adversarial perturbations by leveraging the gradient of the loss function with respect to the input, addressing the computational inefficiencies of earlier approaches. Projected Gradient Descent (PGD) \cite{madry2017towards} enhances attack effectiveness through iterative gradient-based updates, projecting perturbed inputs back into the constrained space after each step. Wong et al. \cite{wong2020fast} improved attack efficiency by introducing random initialization points in FGSM-based attacks. Schwinn et al. \cite{schwinn2023exploring} increased attack diversity by injecting noise into the output while mitigating gradient obfuscation caused by low-confidence predictions. Beyond standard white-box attacks, various techniques have been proposed to improve adversarial transferability \cite{wang2021admix, dong2018boosting}. These include random input transformations \cite{xie2019improving}, translation-invariant perturbation aggregation \cite{dong2019evading}, and substituting momentum-based gradient updates with Nesterov accelerated gradients \cite{lin2019nesterov}.

Adversarial attacks in deep reinforcement learning (DRL) have been widely studied, revealing critical vulnerabilities in agent policies \cite{ilahi2021challenges, duan2025rethinking}. Huang et al. \cite{huang2017adversarial} demonstrated that policy-based reinforcement learning (RL) agents are highly susceptible to adversarial perturbations on state observations, showing that FGSM attacks can significantly degrade performance in Atari 2600 games. Pattanaik et al. \cite{pattanaik2017robust} introduced adversarial examples by computing gradients of the critic network with respect to states and integrated them into the training of Deep Double Q-Network (DDQN) and Deep Deterministic Policy Gradient (DDPG), enhancing robustness. Lin et al. \cite{lin2017tactics} proposed strategically timed attacks that selectively perturb key decision-making states, achieving high attack success rates with minimal perturbations. Recent work has shifted towards theoretical modeling of adversarial attacks within the Markov Decision Process (MDP) framework. Weng et al. \cite{weng2019toward} introduced a systematic evaluation framework for RL robustness in continuous control, defining two primary threat models: observation manipulations and action manipulations. Zhang et al. \cite{zhang2020robust} proposed SA-MDP, which provides a theoretical foundation for modeling state adversarial attacks within MDPs. Oikarinen et al. \cite{oikarinen2021robust} developed RADIAL-RL, a general framework for training RL agents to enhance resilience against adversarial attacks, and introduced Greedy Worst-Case Reward as a new evaluation metric for agent robustness.

\section{System Model}
\subsection{Deep Reinforcement Learning}\label{sec:drl}
DRL problems are formalized as Markov Decision Processes (MDPs), which provide a theoretical framework for modeling sequential decision-making in unknown environments \cite{sutton2018reinforcement}. A MDP is defined as a tuple $\left \langle \mathcal{S},\mathcal{A},R,\mathcal{P},\gamma \right \rangle $, where $\mathcal{S}$ and $\mathcal{A}$ denote the state and action spaces, respectively. The reward function $R: \mathcal{S} \times \mathcal{A} \to \mathcal{R} $ assigns a scalar reward to each state-action pair, while the transition function $\mathcal{P}: \mathcal{S} \times \mathcal{A} \times \mathcal{S} \to \left [ 0,1 \right ] $ defines the probability of transitioning between states given an action. The discount factor $\gamma \in \left [ 0,1 \right )$ governs the weight of future rewards. A stochastic policy $\pi : \mathcal{S} \times \mathcal{A} \to \left [ 0,1 \right ]$ defines a probability distribution over actions given a state, where $a\sim \pi \left ( \cdot \mid s \right ) $ denotes sampling an action $a$ in state $s$. At each time step $t$, the agent selects an action $a_t\sim \pi \left ( \cdot \mid s_t \right )$ based on the current state $s_t$. The environment returns a reward $R\left ( s_t,a_t \right )$ and transitions to the next state $s_{t+1}\sim \mathcal{P} \left ( \cdot \mid s_t,a_t \right )$. The agent's objective is to learn a policy $\pi$ that maximizes the expected discounted return:
\begin{equation}
\label{eq:1}
\begin{aligned}
J\left ( \pi  \right ) = {\textstyle \sum_{t=0}^{\infty }\gamma ^t\mathbb{E}_{a\sim \pi,s\sim \mathcal{P} } \left [ R\left ( s_t,a_t \right )  \right ] }.
\end{aligned}
\end{equation}



To find the optimal policy, value-based reinforcement learning methods are widely used to derive optimal policies. These methods are advantageous due to their ability to efficiently estimate long-term rewards through iterative value estimation, enabling effective decision-making in complex environments. The state value function is defined as $V^{\pi } \left ( s_t \right ) = \mathbb{E}_{a\sim \pi,s\sim \mathcal{P} } \left [ {\textstyle \sum_{k=0}^{\infty } \gamma ^k R\left ( s_{t+k},a_{t+k} \right )}  \right ]  $ which represents the expected discounted return starting from state \( s_t \) under policy \( \pi \). Similarly, the action-value function (Q-function) that quantifies the expected discounted return for selecting action \( a_t \) in state \( s_t \) under policy \( \pi \) is given by  $Q^{\pi } \left ( s_t,a_t \right ) = R\left ( s_t,a_t \right )+\mathbb{E}_{s_{t+1}\sim \mathcal{P} } \left [ V^{\pi } \left ( s_{t+1} \right )   \right ]  $ 
Policy updates rely on the value function or Q-function. For instance, in Deep Q-Networks (DQN) \cite{mnih2013playing}, the policy follows a greedy selection, i.e., $\pi = {\mathrm{arg} \max}_{a} Q^\pi\left ( s,a \right )$, and the Q-function is updated as
\begin{equation}
\label{eq:1}
\footnotesize
\begin{aligned}
Q^\pi _\phi \!\left ( \!s_t,a_t \right )\!\!\gets \! Q^\pi _\phi\left ( s_t,a_t \right ) \!+\! l\left [ R\left ( s_t,a_t \right ) \!+\gamma Q^\pi _\phi\left ( s_{t+1},a_{t+1} \!\right )\!-\! Q^\pi _\phi\left ( s_t,a_t \right )\right]\!\!, \\
\end{aligned}
\end{equation}
where \( l \) is the learning rate, and \( Q^\pi _\phi \) is the Q-function parameterized by a deep neural network with parameters \( \phi \). 
In the remainder of the paper, we assume all Q-functions and policies are parameterized. For brevity, we denote \( Q^\pi _\phi \) as \( Q^\pi \) and the policy \( \pi _\varphi \) as \( \pi \), where \( \varphi \) represents the policy network parameters.


A fundamental aspect of policy optimization is the theoretical bound on the difference in expected returns between two policies, which provides critical insights into policy improvement guarantees and stability. Let $A^\pi \left ( s,a \right ) = Q^\pi \left ( s,a \right )-V^\pi \left ( s \right )$ denote the advantage of selecting action $a$ in state $s$ under policy $\pi$. The discounted future state distribution of all possible trajectories starting from state $s$ under $\pi$ is represented as $d^\pi \left ( s \right ) =\left ( 1-\gamma  \right )  {\textstyle \sum_{k=0}^{\infty }\gamma ^k P\left ( s_k=s\mid \pi  \right ) } $. 

Constrained Policy Optimization (CPO) \cite{achiam2017constrained} establishes the relationship between the performance difference of two policies and their state-action visitation dynamics under a reward function. Specifically, CPO provides the following bounds for the performance difference:
\begin{equation}
\label{eq:CPO}
 D^-_{\pi ,f}\left ( \pi ' \right ) \le J\left ( \pi' \right ) \!-\!J\left ( \pi  \right )  \le D^+_{\pi ,f}\left ( \pi ' \right ),
\end{equation}
where $D^\pm _{\pi ,f}\left ( \pi ' \right ) =\frac{L_{\pi ,f}\left ( \pi ' \right ) }{1-\gamma } \pm \frac{2\gamma \xi^{\pi '} _f}{\left ( 1-\gamma  \right )^2 }\mathbb{E}_{s\sim d^\pi}\left [ \mathrm{D}_\mathrm{TV}\left ( \pi \parallel   \pi '    \right ) \left [ s \right ] \right ]$. For any function $f: \mathcal{S} \to \mathcal{R} $ and policies $\pi$, $\pi '$, $L_{\pi ,f}\left ( \pi ' \right )$, and $\xi^{\pi '} _f$ are given by the following definitions:
\begin{align}
& L_{\pi ,f}\left ( \pi ' \right ) =\mathbb{E}_{s\sim d^\pi, a\sim\pi,s'\sim \mathcal{P} }\left [ \left ( \frac{\pi '\left ( a\mid s \right ) }{\pi\left ( a\mid s \right )}-1  \right ) \delta _f\left ( s,a,s' \right ) \right ] , \\
& \xi^{\pi '} _f=\max_s \left | \mathbb{E}_{a\sim\pi ',s'\sim \mathcal{P} } \left [ \delta _f\left ( s,a,s' \right )  \right ]  \right |  ,
\end{align}
where $\delta _f\left ( s,a,s' \right )=R\left ( s,a,s' \right )+\gamma f\left ( s' \right ) -f\left ( s\right )$.

\subsection{White-box Adversarial Attack}

White-box adversarial attacks are common methods for evaluating the vulnerability of deep neural networks (DNNs) by leveraging gradient-based perturbations to manipulate input samples~\cite{chakraborty2018adversarial}. To be Specific, given an input sample $x$ with corresponding ground truth label $y$,  a DNN is defined as a function $f: \mathcal{X} \to \mathcal{Y} $, where $\mathcal{X}$ and $\mathcal{Y}$ denote the input and output spaces, respectively. The objective of an adversarial attack is to introduce a minimal perturbation $\eta$ to $x$, such that the model produces an incorrect prediction for the perturbed input $x^*$. This can be formulated as a constrained optimization problem:
\begin{equation}
\label{eq:2}
{\mathrm{arg} \max}_{x^*} J\left ( x^*,y \right ) ,\quad \mathrm{s.t.} \ \left \| \eta \right \| _p\le \epsilon,
\end{equation}
where $x^*=x+ \eta $ represents the adversarial example, $\eta$ denotes the adversarial perturbation, $J\left ( \cdot ,\cdot  \right ) $ is the loss function, typically mean squared error or cross-entropy. The constraint $\ \left \| \eta \right \| _p\le \epsilon$ ensures that the perturbation magnitude remains within a predefined threshold $\epsilon\in \left ( 0,\infty  \right ) $, where $\ \left \| \cdot  \right \| _p$ denotes the $L_p$ norm. The choice of $\ \left \| \cdot  \right \| _p$ depends on the application: the $L_\infty$ norm is commonly used for image-based tasks (e.g., Arcade Learning Environment \cite{bellemare2013arcade}), whereas the $L_2$ norm is often preferred for physical-state-based environments (e.g., Safety Gymnasium environment \cite{ji2023safety}).

To solve the constrained optimization problem in Eq.~\ref{eq:2}, FGSM\cite{goodfellow2014explaining} approximates the loss function linearly and generates adversarial examples using a single-step update:
\begin{equation}
\label{eq:3}
x^*=x + \epsilon \cdot \mathrm{sign}\left ( \nabla _x J\left ( x,y \right )  \right ),
\end{equation}
where $\mathrm{sign}\left ( \cdot  \right ) $ represents the sign function. The perturbation in Eq.~\ref{eq:3} satisfies the $L_\infty$ norm constraint. For the $L_2$ norm-setting, $\epsilon$ is adjusted as $\epsilon / \left \| \nabla _x J\left ( x,y \right ) \right \| _2$ to ensure the perturbation magnitude remains within the constraint. The PGD~\cite{madry2017towards} attack extends FGSM by applying iterative gradient updates with a step size $\delta $ while projecting the perturbed input back into the feasible $\epsilon$-ball:
\begin{equation}
\label{eq:4}
x^*_{n+1}=x^*_n + \delta  \cdot \mathrm{sign}\left ( \nabla _x J\left ( x^*_n,y \right )  \right ), 
\end{equation}
where $n=0,1,\dots ,N$, and $N$ denotes the number of iterations, The initial perturbation is set as $x^*_0=x+\varepsilon \cdot \mathcal{N} \left ( \mathbf{0},\mathbf{I}   \right )$, where $\varepsilon$ is a scaling factors, and $ \mathcal{N} \left ( \mathbf{0},\mathbf{I}   \right )$ represents a multivariate standard normal distribution. To enforce the constraint in Eq.~\ref{eq:2}, the step size is typically set to $\delta = \epsilon / N$.

\section{Theoretical Analysis of Attack Success Conditions}\label{sec:ta}

\subsection{Adversarial Victim-Dynamics Markov Decision Process}
We formalize the execution of DRL agents under adversarial attacks by establishing an Adversarial Victim-Dynamics Markov Decision Process (AVD-MDP). The AVD-MDP is defined as a tuple: 
\begin{equation}
\left \langle \mathcal{S},\mathcal{A},\mathcal{B},\mu,R,\mathcal{P},\gamma \right \rangle,
\end{equation}
where $\mathcal{B}$ denotes the space of allowable adversarial perturbations, and $\mu \left ( \cdot \mid s \right ) $ represents the policy of the victim agent (i.e., the agent under attack). The remaining elements follow the standard MDP formulation introduced in Section~\ref{sec:drl}. Let $\nu \left ( \cdot \mid s,\mu  \right ) $ denote the adversarial policy (i.e., adversarial agent's policy), representing the attacker's strategy under a white-box setting, where the attacker has full access to the victim's policy $\mu$ and generates adversarial perturbations based on both the current state $s$ and $\mu$. Notably, $\nu \left ( \cdot \mid s,\mu  \right ) $ is not necessarily a policy in the conventional DRL sense but can be derived by solving the constrained optimization problem in Eq.~\ref{eq:2}, using methods such as FGSM or PGD. At each time step $t$, the adversarial policy samples a perturbation $\eta _t\sim \nu \left ( \cdot \mid s_t,\mu  \right ) $  which is then applied to the victim agent’s observation. The victim agent selects an action according to its policy
\begin{equation}
a _t \sim \mu  \left ( \cdot \mid s_t+ \eta_t \right ),
\end{equation}
where $s_t+ \eta_t$ is the perturbed state.The single attack process is independent of the environment dynamics, forming a one-step sequential decision process. To formalize adversarial attacks in AVD-MDP, we introduce the following definition:
\begin{definition} {\bf {(Single-Step Adversarial Attack)}}
Let $\mu \left ( \cdot \mid s \right ) $ denote the victim agent’s policy and $\nu \left ( \cdot \mid s , \mu  \right )$ represents the adversarial policy. We define a single-step attack $\mu \oplus \nu  \left ( \cdot \mid s \right )$ from the adversarial agent to the victim agent under state $s$ as:
\begin{equation}
\mu \oplus \nu  \left ( \cdot \mid s \right ) = \mu \left ( \cdot \mid s +\eta  \right ),  \eta \sim \nu \left ( \cdot \mid s , \mu  \right ).  
\end{equation}
\end{definition}

In AVD-MDP, adversarial attacks occur after the victim agent has fully converged in the environment, achieving optimal or near-optimal returns in attack-free rollouts. The goal of adversarial agents is to degrade the victim agent’s performance by minimizing its expected return:
\begin{equation}
\label{eq:6}
J\left ( \nu \right ) = \min_{\nu} { \sum_{t=0}^{\infty }\gamma ^t\mathbb{E}_{a\sim \mu \oplus \nu ,s\sim \mathcal{P} } \left [ R\left ( s_t,a_t \right ) \right ] },
\end{equation}
where $R\left ( s_t,a_t \right )$ represents the reward function and $\gamma$ is the discount factor. The adversarial attack modifies the victim agent's state-action trajectory, leading to a distinction between attack-free and attacked rollouts: 
\begin{equation*}
\begin{aligned}
&\phantom{\overset{\nu  }{\rightarrow} \eta _0} \underbrace{s_0\overset{\mu }{\rightarrow} a_0\overset{\mathcal{P}  }{\rightarrow}s_1\overset{\mu }{\rightarrow} \dots \overset{\mathcal{P} }{\rightarrow}s_t \overset{\mu }{\rightarrow}a_t\overset{\mathcal{P}  }{\rightarrow}\dots}_{\text{attack-free rollout}} \\
&\underbrace{s_0\overset{\nu  }{\rightarrow} \eta _0 \overset{\mu }{\rightarrow} a_0\overset{\mathcal{P}  }{\rightarrow}s_1\overset{\mu }{\rightarrow}\dots \overset{\mathcal{P} }{\rightarrow}s_t \overset{\nu  }{\rightarrow} \eta _t \overset{\mu }{\rightarrow}a_t\overset{\mathcal{P}  }{\rightarrow}\dots}_{\text{attacked rollout}}
\end{aligned}
\end{equation*}

Since the victim agent follows a fixed policy 
$\mu$ throughout the attacked rollout, it can be regarded as part of the environment’s dynamics from the adversarial agent’s perspective. However, unlike typical environmental transitions, this component is fully observable and known to the adversarial agent. We define the adversarial agent’s value function and Q-function as the victim agent’s expected return under the attacked rollout distribution:
\begin{definition} {\bf{(Value and Q-functions of the Adversarial Agent)}}
In AVD-MDP, the adversarial agent's value function $V^{\mu \oplus \nu }$ and Q-function $Q^{\mu \oplus \nu  }$ with victim agent's policy $\mu$ and adversarial policy $\nu$ are defined as:
\begin{align}
& V^{\mu \oplus \nu } \left ( s_t \right ) = \mathbb{E}_{a\sim \mu \oplus \nu,s\sim \mathcal{P} } \left [ {\textstyle \sum_{k=0}^{\infty } \gamma ^k R\left ( s_{t+k},a_{t+k} \right )}  \right ], \label{eq:7} \\
& Q^{\mu \oplus \nu  } \left ( s_t,a_t \right ) = R\left ( s_t,a_t \right )+\mathbb{E}_{s_{t+1}\sim \mathcal{P} } \left [ V^{\mu \oplus \nu  } \left ( s_{t+1} \right )   \right ], \label{eq:8}
\end{align}
where $s_t\in \mathcal{S}$, $a_t\in \mathcal{A}$, $\eta _t\in \mathcal{B}$ are state, action, and adversarial perturbation, respectively.
\end{definition}

AVD-MDP reformulates the adversarial attack problem as a single-agent MDP, allowing the adversary to systematically optimize perturbations over long horizons. Different from existing works\cite{weng2019toward,zhang2020robust}, which primarily focus on perturbing individual actions or states, AVD-MDP incorporates the victim agent's policy into the environmental dynamics and directly optimizes for minimizing the victim agent’s expected return. This modeling approach aligns closely with real-world safety testing scenarios, as the expected return directly reflects the victim agent's performance in the targeted task.

\subsection{Analysis of Attack Success Conditions}

To develop effective adversarial policies, we first analyze the theoretical conditions ensuring a successful attack. For an arbitrary state $s$ and victim agent's policy $\mu \left ( \cdot \mid s \right ) $,  a successful attack reduces its expected return:
\begin{equation}
 V^\mu \left ( s \right ) - V^{\mu\oplus \nu } \left ( s \right ) \ge 0, \  \text{for} \ \forall s\in \mathcal{S} \ \text{and} \ \mu \left ( \cdot \mid s \right ),    
\end{equation}
We derive the following theorem to provide a necessary condition for a successful attack.
\begin{theorem} {\bf{(Necessary Condition for Attack Success)}} \label{the:1}
Given a victim agent's policy $\mu \left ( \cdot \mid s \right ) $ and an adversarial policy $\nu  \left ( \cdot \mid s,\mu  \right ) $, a necessary condition for a successful attack is given by:
\begin{equation}
\delta \left ( \mu ,\nu \right ) \left [ s \right ] \le \frac{2\gamma}{1-\gamma}  \cdot   \mathrm{D}_\mathrm{TV}\left ( \mu \parallel \mu \oplus \nu \right ) \left [ s \right ],
\end{equation}
where $\delta \left ( \mu ,\nu  \right ) \left [ s \right ]=\sum _s d^{\mu \oplus \nu} \left ( s \right )\sum _a\left ( \mu \oplus \nu \left ( a\mid s \right ) -\mu \left ( a\mid s \right ) \right )$ denotes the state-wise policy divergence and $\mathrm{D}_\mathrm{TV}\left ( \mu \parallel \mu \oplus \nu \right ) \left [ s \right ]$ represents the total variation distance between the original and attacked policies, and $\gamma$ is the temporal discount factor. 
\end{theorem}
\begin{proof}
We analyze the discrepancy in expected returns between the victim agent’s policy before and after the attack. Our derivation follows the part of prior work on bounding value function differences but is specifically adapted to adversarial attack settings. We begin by considering an upper bound on the value function difference between the victim's original and attacked policies. Building upon the theoretical results on policy performance differences outlined in Eq.~\ref{eq:CPO}, we derive the following bounds:
\begin{equation*}
\begin{aligned}
& V^\mu \left ( s \right ) - V^{\mu\oplus \nu } \left ( s \right ) \\
& \le \underbrace {\frac{1}{1-\gamma } \mathbb{E}_{s\sim d^{\mu \oplus \nu},a\sim \mu \oplus \nu   }\left [\left ( \frac{\mu  \left ( a\mid s \right ) }{\mu \oplus \nu  \left ( a\mid s \right )} -1 \right )R\left ( s,a \right ) \right ] }_{\text{first term of r.h.s.}}\\
& \underbrace {+ \frac{2\gamma }{\left ( 1-\gamma  \right )^2 } \max_{s,a\sim \mu } \left | R\left ( s,a \right ) \right |   \mathbb{E} _{s\sim d^{\mu \oplus \nu}}\left [ \mathrm{D}_\mathrm{TV}\left ( \mu \parallel   \mu \oplus \nu    \right ) \left [ s \right ] \right ]}_{\text{second term of r.h.s.}}, 
\end{aligned}
\end{equation*}
where the first term measures the direct impact of the policy shift on expected rewards, while the second term quantifies the uncertainty introduced by policy divergence. 

We assume that the reward function is independent of the next state, i.e., $R: \mathcal{S} \times \mathcal{A} \to \mathcal{R}$, and assume function $f:\mathcal{S} \to \left \{ 0 \right \} $ to simplify the derivation. Under these assumptions, we bound the first term of the right-hand side  as
\begin{flalign*}
& {\frac{1}{1-\gamma } \mathbb{E}_{s\sim d^{\mu \oplus \nu},a\sim \mu \oplus \nu   }\left [\left ( \frac{\mu  \left ( a\mid s \right ) }{\mu \oplus \nu  \left ( a\mid s \right )} -1 \right )R\left ( s,a \right ) \right ] } \\
& = \frac{1}{1-\gamma } \sum_{s} d^{\mu \oplus \nu}\left ( s \right )  \mathbb{E}_{a\sim \mu \oplus \nu  }\left [\left ( \frac{\mu  \left ( a\mid s \right ) }{\mu \oplus \nu  \left ( a\mid s \right )} -1 \right )R\left ( s,a \right ) \right ] \\
& = \frac{1}{1-\gamma } \sum_{s} d^{\mu \oplus \nu}\left ( s \right ) \sum_{a} \left [ \mu  \left ( a\mid s \right ) -\mu \oplus \nu  \left ( a\mid s \right )  \right ]R\left ( s,a \right ) \\
& \le - \frac{1}{1-\gamma } \max_{s,a\sim \mu \oplus \nu } \left | R\left ( s,a \right ) \right |  \delta \left ( \mu ,\nu  \right ) \left [ s \right ]. \\
\end{flalign*}
Similarly, we bound the second term of the right-hand side  as
\begin{flalign*}
\small 
& {\frac{2\gamma }{\left ( 1-\gamma  \right )^2 } \max_{s,a\sim \mu } \left | R\left ( s,a \right ) \right |   \mathbb{E} _{s\sim d^{\mu \oplus \nu}}\left [ \mathrm{D}_\mathrm{TV}\left ( \mu \parallel   \mu \oplus \nu    \right ) \left [ s \right ] \right ]} & \\
& = \!\frac{2\gamma }{ 1-\gamma  }\! \max_{s,a\sim \mu   } \left | R\left ( s,a \right )  \right | {\sum_{k=0}^{\infty }\!\gamma ^k P\!\left ( s_k\!=\!s\!\mid \!\mu \!\oplus\! \nu    \right ) \mathrm{D}_\mathrm{TV}\left ( \mu \parallel   \mu \oplus \nu    \right ) \!\left [ s \right ] } & \\
& \le \frac{2\gamma }{ \left ( 1-\gamma \right )^2 } \max_{s,a\sim \mu   } \left | R\left ( s,a \right )  \right | \mathrm{D}_\mathrm{TV}\left ( \mu \parallel   \mu \oplus \nu    \right ) \left [ s \right ]. &
\end{flalign*}

Since $ \max_{s,a\sim \mu \oplus \nu } \left | R\left ( s,a \right ) \right | =  \max_{s,a\sim \mu   } \left | R\left ( s,a \right )  \right |=\left |R_{\mathrm{max}}\right |$, we can derive an upper bound on the change in value function under pre- and post-attack:
\begin{equation*}
\begin{aligned}
&  V^\mu \left ( s \right ) - V^{\mu\oplus \nu } \left ( s \right ) \\
& \le \frac{1}{1-\gamma } \left |R_{\mathrm{max}}\right |\left [ \frac{2\gamma}{1-\gamma }\mathrm{D}_\mathrm{TV}\left ( \mu \parallel   \mu \oplus \nu    \right ) \left [ s \right ]- \delta \left ( \mu ,\nu  \right ) \left [ s \right ]\right ]. 
\end{aligned}
\end{equation*}

In DRL, adversarial attacks aim to reduce the expected return of the victim agent, satisfying $ V^\mu \left ( s \right ) - V^{\mu\oplus \nu } \left ( s \right ) \ge 0$. Thus, we can obtain:
\begin{equation*} \frac{1}{1-\gamma } \left |R_{\mathrm{max}}\right |\left [ \frac{2\gamma}{1-\gamma }\mathrm{D}_\mathrm{TV}\left ( \mu \parallel   \mu \oplus \nu    \right ) \left [ s \right ]- \delta \left ( \mu ,\nu  \right ) \left [ s \right ]\right ] \ge 0.
\end{equation*}
Then, we rearrange terms, yields
\begin{equation*}
\begin{aligned}
\delta \left ( \mu ,\nu  \right ) \left [ s \right ] \le \frac{2\gamma}{1-\gamma }\mathrm{D}_\mathrm{TV}\left ( \mu \parallel   \mu \oplus \nu    \right ) \left [ s \right ].
\end{aligned}
\end{equation*}
This concludes the proof.
\end{proof}
{\bf Theorem~\ref{the:1}} derives a necessary condition for successful adversarial attacks and reveals a key property of effective adversarial strategies: the stealthiness of perturbations. To be specific, it connects the cumulative changes in the action selections of the victim agent’s policy $\delta(\mu, \nu)$ with the divergence in the victim’s policy distributions before and after the attack, quantified by the total variation distance $\mathrm{D}_\mathrm{TV}\left ( \mu \parallel \mu \oplus \nu \right ) \left [ s \right ]$ in the AVD-MDP process. {\bf Theorem~\ref{the:1}} indicates that the sum of deviations induced by the adversarial policy through state perturbations must bounded by the total variation distance for the attack to succeed.

ADV-MDP can represent a white-box attack process by generating adversarial examples according to Eq.~\ref{eq:2}, where the adversarial agent has full knowledge of the victim's policy. This enables the adversarial agent to strategically optimize perturbations to maximize the policy shift while ensuring the perturbation remains within a predefined constraint. Consequently, the total variation distance between the victim’s pre- and post-attack policies is determined by the constraint and the loss function, which can be approximated as a predefined constant. This leads to an intuitive conclusion consistent with {\bf Theorem~\ref{the:1}}: if the perturbation is too large, the attacked states may be misclassified as distinct, causing the victim to select different actions to evade the attack. For the attack to be effective, the adversarial policy must precisely induce meaningful shifts in the victim’s policy, while keeping perturbations minimal. In other words, the success of an attack hinges on the delicate balance between stealthiness and interference, maximizing control over the victim’s decision-making and ensuring under limited perturbation constraints.

To further refine this analysis, we explore a sufficient condition for a successful attack under a stronger assumption: that the victim agent is capable of receiving negative rewards in the environment. Although this assumption is more restrictive, it allows us to derive a concrete condition that guarantees the success of adversarial attacks, as presented in the following theorem.


\begin{theorem}\label{the:2} {\bf{(Sufficient Condition for Attack Success)}} 
Assume the victim agent can receive negative rewards $\min_{s,a} R\left ( s,a \right ) < 0$. Given a victim agent's policy $\mu \left ( \cdot \mid s \right ) $ and an adversarial agent's policy $\nu  \left ( \cdot \mid s,\mu  \right ) $, let $\delta \left ( \mu ,\nu  \right ) \left [ s \right ]=\sum _s d^{\mu \oplus \nu} \left ( s \right )\sum _a\left ( \mu \oplus \nu \left ( a\mid s \right ) -\mu \left ( a\mid s \right )\right )$, a sufficient condition to successfully attack the victim agent is given by
\begin{equation}
\delta \left ( \mu ,\nu  \right ) \left [ s \right ] \ge  \kappa \cdot  \max_{s}P\left ( s\mid \mu \oplus \nu  \right ) \mathrm{D}_\mathrm{TV}\left ( \mu \parallel   \mu \oplus \nu    \right ) \left [ s \right ],   
\end{equation}
where $\delta \left ( \mu ,\nu  \right ) \left [ s \right ]=\sum _s d^{\mu \oplus \nu} \left ( s \right )\sum _a\left ( \mu \oplus \nu \left ( a\mid s \right ) -\mu \left ( a\mid s \right )\right )$, $\kappa = - 2\gamma\left |R_{\mathrm{max}}\right| /\left (  1-\gamma\right )R_{\mathrm{min}}$ is a constant and $\max_{s}P\left ( s\mid \mu \oplus \nu  \right )$ is the upper bound of the state-visiting probability during the attack.
\end{theorem}
\begin{proof}
The proof proceeds similarly to Theorem~\ref{the:1}, utilizing the lower bounds on expected returns for two arbitrary policies as established in Eq.~\ref{eq:CPO}. We have:
\begin{equation*}
\begin{aligned}
& V^\mu \left ( s \right ) - V^{\mu\oplus \nu } \left ( s \right ) \\
& \ge  \underbrace {\frac{1}{1-\gamma } \mathbb{E}_{s\sim d^{\mu \oplus \nu },a\sim \mu \oplus \nu }\left [\left ( \frac{\mu  \left ( a\mid s \right ) }{\mu \oplus \nu  \left ( a\mid s \right )} -1 \right )R\left ( s,a \right ) \right ] }_{\text{first term of r.h.s.}}\\
& \underbrace {- \frac{2\gamma }{\left ( 1-\gamma  \right )^2 } \max_{s,a\sim \mu  } \left | R\left ( s,a \right ) \right |   \mathbb{E} _{s\sim d^{\mu \oplus \nu }}\left [ \mathrm{D}_\mathrm{TV}\left ( \mu \oplus \nu \parallel   \mu    \right ) \left [ s \right ] \right ]}_{\text{second term of r.h.s.}} 
\end{aligned}
\end{equation*}
Assuming the reward function $R: \mathcal{S} \times \mathcal{A} \to \mathcal{R}$ is independent of the next state, and $\min_{s,a} R(s,a) \ge 0$, we derive from the right-hand terms:
\begin{flalign*}
& {\frac{1}{1-\gamma } \mathbb{E}_{s\sim d^{\mu \oplus \nu },a\sim \mu \oplus \nu }\left [\left ( \frac{\mu  \left ( a\mid s \right ) }{\mu \oplus \nu  \left ( a\mid s \right )} -1 \right )R\left ( s,a \right ) \right ] } \\
& = \frac{1}{1-\gamma } \mathbb{E}_{s\sim d^{\mu \oplus \nu },a\sim \mu \oplus \nu }\left [\left ( \frac{\mu  \left ( a\mid s \right ) }{\mu \oplus \nu  \left ( a\mid s \right )} -1 \right )R\left ( s,a \right ) \right ] \\
& = \frac{1}{1-\gamma } \sum_{s} d^{\mu \oplus \nu }\left ( s \right ) \sum_{a} \left [ \mu  \left ( a\mid s \right ) -\mu \oplus \nu  \left ( a\mid s \right )  \right ]R\left ( s,a \right ) \\
& \ge - \frac{1}{1-\gamma } \min_{s,a\sim {\mu \oplus \nu}}  R\left ( s,a \right )  \delta \left ( \mu ,\nu  \right ) \left [ s \right ]. 
\end{flalign*}

By combining both terms, we can derive a lower bound on the change in value function under pre- and post-attack:
\begin{equation*}
\small
\begin{aligned}
\footnotesize 
& - \frac{2\gamma }{\left ( 1-\gamma  \right )^2 } \max_{s,a\sim \mu  } \left | R\left ( s,a \right ) \right |   \mathbb{E} _{s\sim d^{\mu \oplus \nu }}\left [ \mathrm{D}_\mathrm{TV}\left ( \mu \oplus \nu \parallel   \mu    \right ) \left [ s \right ] \right ] & \\
& = -\frac{2\gamma }{ 1-\gamma} \!\! \max_{s,a\sim \mu   } \left | R\left ( s,a \right )  \right |\!\! \!\!\!\!\phantom{=}{\sum_{k=0}^{\infty }\!\gamma ^k P\!\left ( s_k=s\!\mid\! \mu \!\oplus \!\nu  \!  \right ) \mathrm{D}_\mathrm{TV}\left ( \mu \!\parallel \!  \mu \oplus \!\nu \!   \right ) \!\left [ s \right ] } & \\
& \ge - \frac{2\gamma }{ \left ( 1-\gamma \right )^2 } \max_{s,a\sim \mu   } \left | R\left ( s,a \right )  \right | \max_s P\left ( s\mid \mu \!\oplus\! \nu  \right ) \phantom{\ge} \!\!\!\!\mathrm{D}_\mathrm{TV}\left ( \mu \!\parallel   \!\mu \!\oplus\! \nu    \right ) \!\left [ s \right ].
\end{aligned}
\end{equation*}

Given $ V^\mu \left ( s \right ) - V^{\mu\oplus \nu } \left ( s \right ) \ge 0$, we can derive a sufficient condition for a successful attack:
\begin{equation*}
\footnotesize
\begin{aligned}
 - \frac{{{R_{{\rm{min}}}}}}{{1 - \gamma }}\delta \left( {\mu ,\nu } \right)\left[ s \right] &\ge \frac{{2\gamma }}{{{{\left( {1 - \gamma } \right)}^2}}}\left| {{R_{{\rm{max}}}}} \right|{\max _s}P\left( {s\!\mid\! \mu  \!\oplus\! \nu } \right){{\rm{D}}_{{\rm{TV}}}}\left( {\mu \!\parallel\! \mu  \oplus \nu } \right)\left[ s \right] \\
\delta \left( {\mu ,\nu } \right)\left[ s \right] &\ge \frac{{ - 2\gamma \left| {{R_{{\rm{max}}}}} \right|}}{{\left( {1 - \gamma } \right){R_{{\rm{min}}}}}}{\max _s}P\left( {s\mid \mu  \oplus \nu } \right){{\rm{D}}_{{\rm{TV}}}}\left( {\mu \parallel \mu  \oplus \nu } \right)\left[ s \right],
\end{aligned}
\end{equation*}
where $R_{\mathrm{min}}=\min_{s,a\sim \mu \oplus \nu }  R\left ( s,a \right ) $ represents the minimum reward during attacked rollouts, and $\left |R_{\mathrm{max}}\right |=\max_{s,a\sim \mu   } \left | R\left ( s,a \right )  \right |$ represents the maximum absolute reward during attack-free rollouts. 
\end{proof}

{\bf{Theorem~\ref{the:2}}} reveals a sufficient condition for a successful adversarial attack, highlighting a crucial property of attack strategies: dispersion in state visitation. An attack succeeds if the cumulative policy changes during the attack, denoted by $\delta(\mu, \nu)[s]$, surpasses the product of the maximum state-visiting probability $\max_s P(s | \mu \oplus \nu)$ and the total variation distance $\mathrm{D}_\mathrm{TV}(\mu \parallel \mu \oplus \nu)[s]$. In essence, if the adversarial policy can reduce visitation probabilities for certain states, making the victim agent's state distribution more dispersed, the attack is significantly enhanced. An important insight from {\bf{Theorem~\ref{the:2}}} is that adversarial policies should manipulate the victim agent's state distribution not only by increasing the policy changes but also by dispersing the state visitation probabilities. This dispersion prevents the victim agent from focusing on specific states, complicating its ability to maintain an optimal policy. By exploiting vulnerabilities across a broader range of states, the adversarial policy increases attack success rates.

\begin{figure*}[ht]
    \centering
    \includegraphics[width=0.6\linewidth]{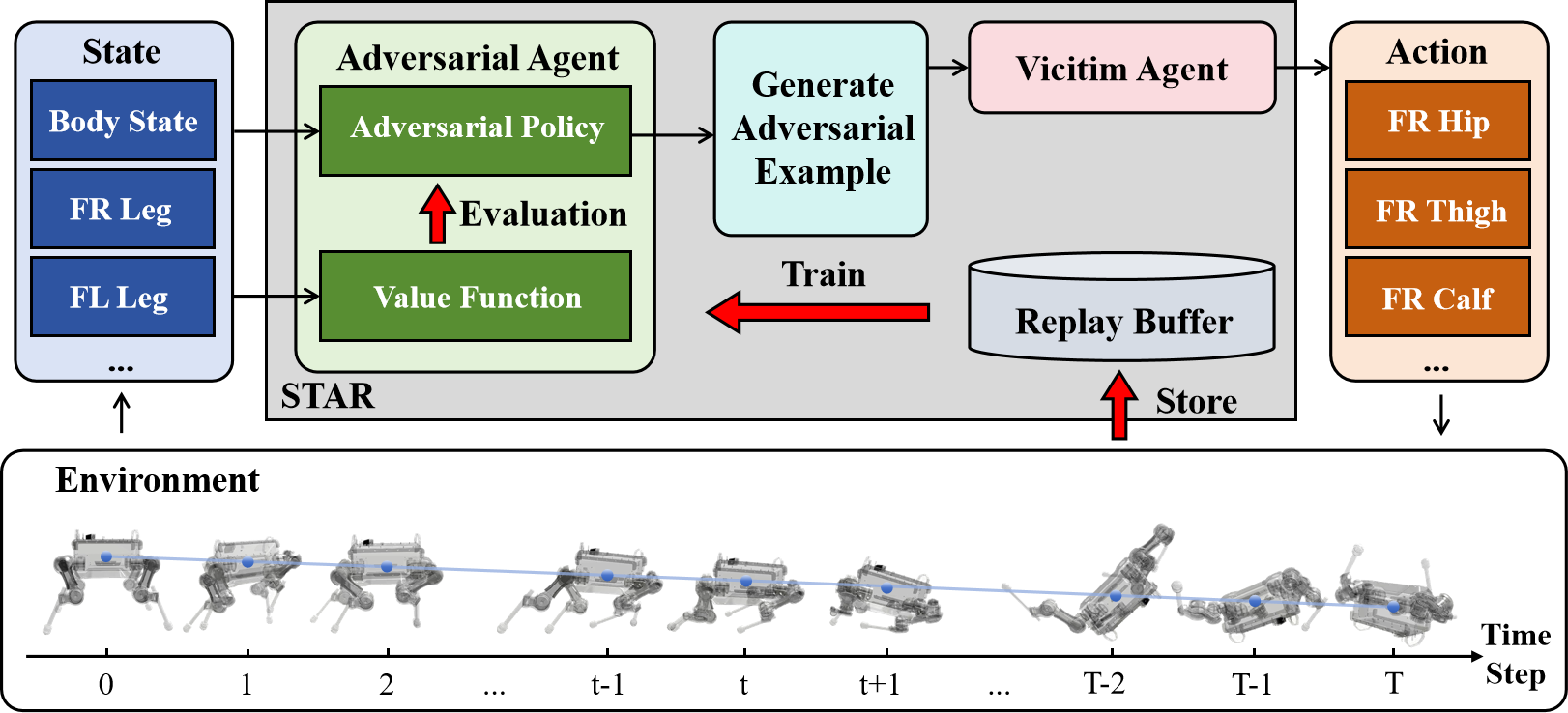}
    \caption{The workflow of STAR.} 
    \label{fig:Overview.}
\end{figure*}

\section{The STAR Framework}
To develop a practical attack algorithm, we analyze the necessary and sufficient conditions for successful attacks in Section~\ref{sec:ta} and derive the following insights:

\begin{itemize}
\item \textbf{Stealthiness of Perturbations}: The adversarial policy must strike a delicate balance between minimizing the expect return of victim's policy while maintaining the perturbation constrained condition. As shown in {\bf{Theorem~\ref{the:1}}}, the cumulative changes in the victim's action selections must be bounded by the total variation distance between policy distributions. This implies that effective attacks necessitates precise and stealthy perturbations to avoid triggering unintended policy responses from the victim agent.

\item \textbf{Dispersion in State Visitation}: {\bf{Theorem~\ref{the:2}}} reveals a substantial improvement in attack effectiveness when dispersing the victim agent's state-visitation distribution. To achieve this, the adversarial policy should strategically redistribute visitation probabilities, preventing the victim agent from concentrating its behavior on specific states. This dispersed approach not only makes it harder for the victim agent to maintain an optimal policy but also ensures more effective utilization of the perturbation budget under the same constraints.
\end{itemize}

To enforce stealthiness and dispersion, we propose a DRL-based white-box attack method called \textbf{S}elective S\textbf{T}ate-\textbf{A}ware \textbf{R}einforcement adversarial attack (STAR). STAR employs a soft mask-based state-targeting mechanism to enhance the stealthiness of the attack by minimizing perturbations on redundant state dimensions. Furthermore, it optimizes an information-theoretic objective that maximizes mutual information between adversarial perturbations, environmental states, and victim actions to ensure a dispersed state-visitation distribution for the victim agent under attack. By integrating these strategies, STAR effectively induces the victim agent into vulnerable states while concentrating perturbations on critical dimensions, maintaining stealthy attack through carefully constrained perturbations. As illustrated in Figure~\ref{fig:Overview.}, STAR first extracts state information from the environment (e.g., robot body state and front right/left leg states). The adversarial agent, consisting of an adversarial policy and a value function, evaluates and generates adversarial examples targeting the victim agent. These perturbations are then injected into the victim agent, systematically disrupting its decision-making process and consequently inducing sub-optimal action selections. STAR employs a reinforcement learning paradigm with an experience replay mechanism, enabling continuous optimization of the adversarial policy and value function during training.

\subsection{Optimization Objective}

The stealthiness of perturbations necessitates gradient-based attack methods, where sign gradients remain an optimal choice. This is because sign gradients maintain the maximum allowable perturbation magnitude in the direction of the steepest descent, ensuring efficient exploitation of the perturbation budget. As a result, STAR adopts adaptive-magnitude perturbations aligned with the sign gradient direction, dynamically adjusting the perturbation magnitude to balance attack effectiveness and stealthiness. Furthermore, STAR incorporates a soft mask function defined as $M_{\mathrm{soft}}(s) = \beta M(s) + (1-\beta)(1-M(s))$ where $M :\mathcal{S} \to \left \{ 0,1 \right \}$ is a binary mask identifying the critical state dimensions, and $\beta \in \left ( 0,1 \right ) $ is an interpolation factor controlling the trade-off between critical and redundant dimensions. The adversarial example generated by the adversarial policy is then formulated as
\begin{equation}
\label{eq:adversarial policy}
\nu \left ( \cdot \mid s,\mu  \right ) =\epsilon \cdot M_{\mathrm{soft}}(s) \cdot \mathrm{sign}\left ( \nabla _s J \left ( s',a \right )  \right ),
\end{equation}
where $s'= s+\varepsilon \cdot \mathcal{N} \left ( \mathbf{0},\mathbf{I}   \right )$, $a\sim \mu \left ( \cdot \mid s\right )$, $\varepsilon \in \left ( 0,1 \right ] $ is the scaling factor, and $\mathcal{N} \left ( \mathbf{0},\mathbf{I}   \right )$ represents a multivariate standard normal distribution that introduces small perturbations to prevent gradient vanishing. Other notations retain their definitions in Eq.~\ref{eq:3} and Eq.~\ref{eq:4}.
As shown in Figure~\ref{fig:2}, adversarial example generation follows a two-branch structure: the magnitude branch and the direction branch. The magnitude branch employs a mask function and an interpolation mechanism to generate a soft mask, dynamically adjusting the perturbation magnitudes while prioritizing critical state dimensions. The direction branch calculates the perturbation direction using Gaussian noise and sign gradients, ensuring effective perturbations along the steepest descent direction. Finally, Finally, the outputs from both branches are combined to generate adversarial perturbations, which are applied to the original state.

The dispersion in state visitation indicates that the optimization objective of STAR necessities an additional term in Eq.~\ref{eq:6}, aiming to strategically diversify the state-visitation distribution of the victim agent. This is formalized using information-theoretic concepts. Let $I\left ( \eta ; s \right ) $ represent the mutual information between adversarial perturbations $\eta$ and states $s$, which quantifies the shared information between perturbations and environmental states. We define $I\left ( \eta ;a\mid s \right )$ as the conditional mutual information between adversarial perturbations $\eta$ and victim agent's actions $a$ in the state $s$. By maximizing $I\left ( \eta ; s \right )$, we ensure that perturbations are closely aligned with current states that have the most significant impact on the victim agent's policy. Simultaneously, maximizing $I\left ( \eta ; a \mid s \right )$ guarantees that the perturbations effectively influence the victim agent's actions based on the current states, thereby disrupting its ability to follow an optimal policy. The combined maximization of total mutual information $I\left ( \eta ; s \right ) +I\left ( \eta ;a\mid s \right ) $ forces the victim agent to deviate from its attack-free state-visitation patterns, leading to a more dispersed distribution and hindering its capacity to maintain optimal performance.

To further enhance the optimization objective of STAR, we minimize the conditional entropy $\mathcal{H}(a|s)$, which reduces the uncertainty in the victim agent's action selection $a$ for the given state $s$. This strengthens dependencies between states and actions, thereby constraining the victim agent's policy to a more focused region within the action space. In summary, STAR maximizes the total mutual information, which could be expressed as:
\begin{equation}
\label{eq:13}
\begin{aligned}
& \phantom{=} I\left ( \eta ; s \right ) +I\left ( \eta ;a\mid s \right ) - \mathcal{H} \left ( a ;s \right ) \\
& = \mathcal{H} \left ( \eta \right ) - \mathcal{H} \left ( \eta\mid s \right ) + \mathcal{H} \left ( \eta\mid s \right )-\mathcal{H} \left ( \eta\mid s,a \right ) - \mathcal{H} \left ( a \mid s \right )\\
& = \mathcal{H} \left ( \eta \right )-\mathcal{H} \left ( \eta\mid s,a \right ) -\mathcal{H} \left ( a\mid s\right )
\end{aligned}
\end{equation}
\noindent where $\mathcal{H} \left ( \cdot  \right ) $ represents the Shannon entropy,  $\mathcal{H} \left ( \eta \right )$ indicates the necessity to maximize the entropy of the adversarial policy's prior distribution. In practice, we adopt a uniform distribution as the prior distribution of the adversarial policy. The terms $-\mathcal{H} \left ( \eta\mid s,a \right )$ and $-\mathcal{H} \left ( a\mid s\right )$ imply that, given state $s$, both the victim agent's and the adversarial agent's policy distributions should exhibit low uncertainty. Consequently, maximizing the objective function presented in Eq.~\ref{eq:13} is equivalent to minimize $\mathcal{H} \left ( \eta\mid s,a \right ) +\mathcal{H} \left ( a\mid s\right )$. Combining the previously stated objective of the adversarial agent from Eq.~\ref{eq:6},  we derive the objective function of STAR as follows:
\begin{equation}
\label{eq:15}
\begin{aligned}
J\left ( \nu \right ) = & \min_{\nu} \sum_{t=0}^{\infty }\gamma ^t \mathbb{E}_{a\sim \mu \oplus \nu ,s\sim \mathcal{P} } \left[ R\left ( s_t,a_t \right ) \right. \\
& \left. + \alpha \left [ \mathcal{H} \left ( \nu\left ( \cdot \mid s_t \right ) \right ) +\mathcal{H} \left ( \mu \oplus \nu\left ( \cdot \mid s_t \right ) \right ) \right ] \right]
\end{aligned}
\end{equation}
where $\alpha \in \left ( 0,1 \right ] $ serves as an entropy weight coefficient that balances two optimization objectives: minimizing the victim agent's return via term $R\left ( s_t,a_t \right )$ and concentrating attacks through dual entropy regularization terms $\mathcal{H} \left ( \nu\left ( \cdot \mid s_t \right ) \right ) +\mathcal{H} \left ( \mu \oplus \nu\left ( \cdot \mid s_t \right ) \right )$.

\begin{figure}[t]
  \centering
  \includegraphics[width=0.48\textwidth]{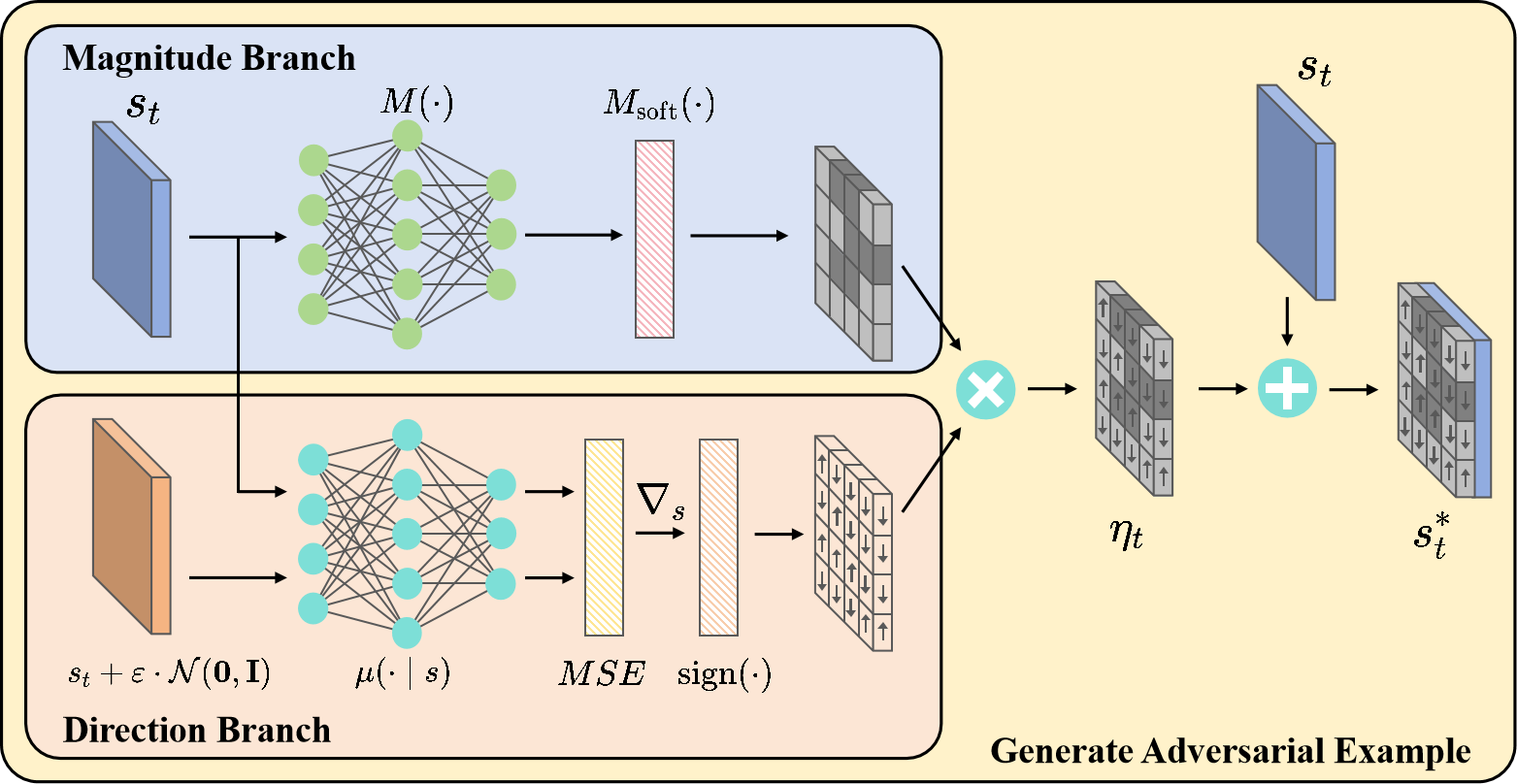}
  \caption{The Adversarial Example Generation Framework of STAR.}
  \label{fig:2}
\end{figure}

\subsection{Training Algorithm}

\begin{algorithm}[t]
\footnotesize
\caption{STAR adversarial attack training procedure}
\label{alg:1}
\begin{flushleft}
\textbf{Input}: victim agent's policy $\mu \left ( \cdot \mid s \right ) $, batch size $N$, entropy weight coefficient $\alpha $, discount factor $\gamma $
\end{flushleft}
\begin{algorithmic}[1] 
\STATE {\bf{Random initialization:}} mask function $M(\cdot )$ with $\theta^M$ and value function 
$V^{\mu \oplus \nu }\left ( \cdot  \right ) $ with $\theta^V$, and replay buffer $\mathcal{D}$
\FOR{each episode}
\STATE Initialize state $s_0$ and $T\gets 0$
\STATE $\eta _0 \sim p\left ( s_0 \right ) $ 
\STATE $a_0 \sim \mu \left ( \cdot \mid s_0+\eta _0 \right )$

\FOR{time step $t$}
\STATE Execute $a_t$, compute reward $r_t=R\left ( s_t,a_t \right ) $, and store transition $\left ( s_t,a_t,r_t,s_{t+1} \right ) $ in $\mathcal{D}$.
\IF{$s_{t+1}$  is terminal}
\STATE $T\gets t+1$
\STATE Calculate $\widehat{R} _t$ and store return in $\mathcal{D}$:
\begin{center} 
$\widehat{R} _t = \sum_{k=0}^{T-t-1} \gamma ^k r_{t+k}+\gamma^{T-t} V^{\mu \oplus \nu }\left ( s_T  \right )$
\end{center}
\STATE  Calculate $\widehat{A} _t$ and store advantage in $\mathcal{D}$:

\begin{center} 
$\widehat{A} _t= \widehat{R}_t -V^{\mu \oplus \nu }\left ( s_t  \right )$
\end{center}
\ENDIF
\STATE $a_t \sim \mu \oplus \nu  \left( \cdot \mid s_t \right )$
\ENDFOR

\FOR{each epoch}
\STATE Sample a batch of  $\left ( s_i,a_i,r_i,s_{i+1},\widehat{R}_i,\widehat{A}_i   \right ) $ from $\mathcal{D}$
\STATE Update the value function $V^{\mu \oplus \nu } $ by minimizing the loss:

$\mathcal{L} \left ( V^{\mu \oplus \nu } \right ) =\frac{1}{N}\sum_{i=0}^{N} \left [ \hat{R}_i -V^{\mu \oplus \nu }\left ( s_i \right )  \right ] ^2$
\STATE Update the mask function $M(\cdot )$ by minimizing the objective:

$J\left ( \nu  \right )  =\frac{1}{N}\sum_{i=0}^{N} \left [ \hat{A}_i +\alpha \log{ \left ( \mu \left ( a_i\mid s_i \right )\nu \left ( \eta _i\mid s_i,\mu  \right )   \right )} \right ]$ 
\ENDFOR
\ENDFOR
\end{algorithmic}
\end{algorithm}

To effectively train the STAR adversarial attack algorithm, we employ an on-policy reinforcement learning framework integrated with Generalized Advantage Estimation (GAE) \cite{schulman2015high}. The on-policy framework is particularly sensitive to subtle variations in the victim agent's behavioral patterns, ensuring that policy updates are consistently derived from the most recent interactions between the attacker and the victim. This approach mitigates the distribution shift issues commonly associated with off-policy methods, which rely on outdated experiences. By leveraging on-distribution samples from current policy trajectories, the on-policy framework facilitates real-time adaptation of perturbation strategies while maintaining update stability. Additionally, STAR utilizes GAE to enhance the efficiency and stability of policy gradient algorithms by providing a more flexible and accurate approximation of the advantage function. The advantage function  $\widehat{A}_t$ quantifies the relative improvement of an action compared to the expected return under the current policy. However, accurate estimation of $\widehat{A}_t$ becomes challenging in environments with sparse or noisy rewards. GAE addresses this challenge by introducing a tunable hyperparameter $\lambda$, which balances short-term rewards and long-term returns while optimizing the bias-variance trade-off. Specifically, the advantage function $\widehat{A}_t$ at time step $t$ is computed as:
\begin{equation}
\begin{aligned}
\widehat{A} _t & = \sum_{k=0}^{T-t-1} \left ( \gamma \lambda  \right ) ^k \delta _{t+k},
\end{aligned}    
\end{equation}
where $\delta _t = r_t+\gamma V^{\mu \oplus \nu }\left ( s_{t+1}  \right )-V^{\mu \oplus \nu }\left ( s_t  \right )$. 
By setting $\lambda = 1$, GAE enables full accumulation of discounted rewards while integrating value function estimates to yield a bias-free advantage function estimator. This setup is particularly suitable for tasks with long-term dependencies, such as adversarial attacks in robotic manipulation, as it captures the complete impact of future rewards without requiring explicit weighting schemes or truncation of temporal difference errors. To quantify state-value expectations under policy ${\mu \oplus \nu}$, STAR employs a parameterized value function $V^{\mu \oplus \nu}$ that is optimized through minimization of the following mean squared error loss:
\begin{equation}
\mathcal{L} \left ( V^{\mu \oplus \nu } \right ) =\mathbb{E} _{ s_i,\hat{R}_i \sim \mathcal{D}  }\left [ \hat{R}_i -V^{\mu \oplus \nu }\left ( s_i \right )  \right ] ^2,
\end{equation}
where $\widehat{R} _i = \sum_{k=0}^{T-i-1} \gamma ^k r_{i+k}+\gamma^{T-i} V^{\mu \oplus \nu }\left ( s_T  \right )$ denotes the expected return and $\mathcal{D}$ represents the replay buffer of on-policy trajectories sampled from the current policy $\mu \oplus \nu$. The mask function $M(s)$ is parameterized by a neural network $\theta^M$. It is optimized by minimizing the objective function in Eq.~\ref{eq:15}, which incorporates both the advantage and entropy regularization terms:
\begin{equation}
J \left ( \nu \right ) =\mathbb{E} _{ s_i,\hat{A}_i \sim \mathcal{D}}\left [ \hat{A}_i +\alpha \log{ \left ( \mu \left ( a_i\mid s_i \right )\nu \left ( \eta _i\mid s_i,\mu  \right )   \right )}  \right ].
\end{equation}

The STAR training algorithm is detailed in Algorithm~\ref{alg:1}. The training process begins by initializing the mask function $M(\cdot)$ with network parameters $\theta^M$, the value function $V^{\mu \oplus \nu}(\cdot)$ with network parameters $\theta^V$ and a replay buffer $\mathcal{D}$ to store transitions (line 1). For each episode, the initial state $s_0$ is set, and the episode length counter $T$ is initialized to zero (line 3). The adversarial perturbation is first sampled from a uniform distribution $p(s_0)$ to maximize the entropy of the adversarial policy’s prior distribution (lines 4-5). During each time step of the episode, the selected action $a_t$ is executed in the environment, yielding a reward $r_t$ and transitioning to the next state $s_{t+1}$. The transition tuple $\left(s_t, a_t, r_t, s_{t+1}\right)$ is stored in the replay buffer $\mathcal{D}$ for subsequent training (line 7). If the episode is not terminal, the adversarial example are generated based on the current state $s_t$ and victim agent's policy, after which the victim agent selects an action $a_t$ according to the adversarial example to continue the interaction in the environment (line 13). If the environment reaches a terminal state $s_{t+1}$, the episode length $T$ is updated. Returns $\widehat{R}_t$ and advantages $\widehat{A}_t$ for all time steps $t$ in the episode are computed and stored in the replay buffer $\mathcal{D}$ (lines 9-11). Training occurs after the collection of trajectories. For each training epoch, a batch of $N$ transitions, returns, and advantages is sampled from the replay buffer $\mathcal{D}$ (line 16). The value function $V^{\mu \oplus \nu}$ is updated by minimizing the mean squared error between the predicted value and the stored returns $\widehat{R}i$ (line 17). The mask function $M(\cdot)$, defined in the adversarial policy in Eq.~\ref{eq:adversarial policy}, is optimized by minimizing the objective $J(\nu)$ (line 18).

\section{Implementation}


\begin{figure}
    \centering
    \includegraphics[width=0.9\linewidth]{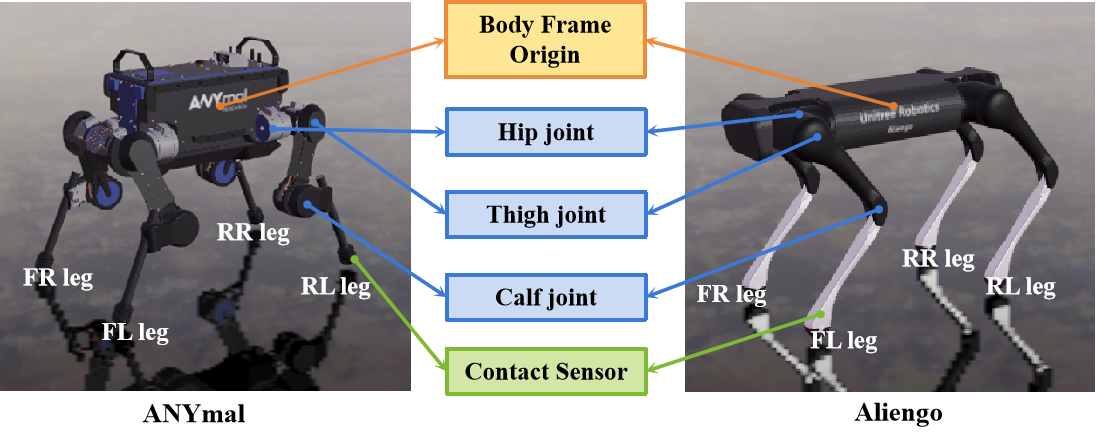}
    \caption{Aliengo and ANYmal, quadruped robots in the Raisim platform \cite{raisim}.}
    \label{fig:robots_demo}
\end{figure}

\begin{table}[]
\renewcommand{\arraystretch}{1.2}
\centering
\caption{State and Action Space Definition}
\begin{tabular}{ll}
\hline
\multicolumn{2}{c}{State Space (34 dimensions)} \\
\hline
Body Height (m) & 0 \\
Body Orientation (rad) & 1:3 \\
Joint Angles (rad) & 3:15 \\
Body Linear Velocity (m/s) & 15:18 \\
Body Angular Velocity (rad/s) & 18:21 \\
Joint Velocities (rad/s) & 21:33 \\
\hline
\multicolumn{2}{c}{Action Space (12 dimensions)} \\
\hline
Front Right Leg (Hip, Thigh, Calf) (rad) & 0:3 \\
Front Left Leg (Hip, Thigh, Calf) (rad) & 3:6 \\
Rear Right Leg (Hip, Thigh, Calf) (rad) & 6:9 \\
Rear Left Leg (Hip, Thigh, Calf) (rad) & 9:12 \\
\hline
\end{tabular}
\label{tab:state_action}
\end{table}

\textbf{Emulation Setup.} 
We evaluate quadrupedal locomotion control and adversarial robustness using the RaiSim platform \cite{raisim}, conducting experiments on two commercial quadruped robots: Aliengo and ANYmal, as shown in Figure~\ref{fig:robots_demo}. Both robots feature a floating base with 6 degrees of freedom (3 for position and 3 for orientation) and 12 actuated joints, with each leg consisting of 3 joints (hip, thigh, and calf). The robots are controlled by a Proportional-derivative controller with feedforward torque, where the proportional gain $P=100.0$ controls the stiffness of position tracking and the derivative gain $D=1.0$ provides damping to reduce oscillations. Each leg’s end-effector has a contact sensor to detect ground interaction. A fall is defined as any non-foot component making contact with the ground.

In each episode, the robot is initialized 0.5 meters above the ground in a standardized stance. The episode runs for a maximum of 4 seconds (400 control steps) or terminates early if a fall occurs. The controller operates at 100 Hz with a control interval of 10 ms, while the physics simulation runs at 400 Hz with a timestep of 2.5 ms. All experiments are conducted on a server equipped with four NVIDIA GeForce RTX 3090 (24GB) GPUs.

\textbf{Victim and Adversarial Agents.}
The victim agent is trained using Proximal Policy Optimization (PPO) \cite{schulman2017proximal}, an on-policy reinforcement learning algorithm that combines trust region optimization with clipped surrogate objectives. It employs a dual-network architecture consisting of an actor network and a critic network. The learning rate is adaptively modulated based on the Kullback-Leibler (KL) divergence between consecutive policy iterations. The network architecture and hyperparameters are summarized in Table \ref{hyper}. The victim agent's reward function balances energy efficiency and performance through torque penalties and forward velocity incentives:
\begin{equation}
R_{\text{vic}}(\tau,v_x) = \zeta  \sum_{i=1}^{12} \tau_i^2 + \kappa \min(v_x, 4.0),
\end{equation}
where $\tau_i$ represents the torque of the $i$-th joint, $v_x$ is the forward velocity, $\zeta = -4 \times 10^{-5}$ is the energy efficiency coefficient penalizing the sum of squared torques, and $\kappa = 0.3$ is the forward velocity coefficient encouraging locomotion with an upper bound of  4.0 m/s. Training spans $1 \times 10^5$ timesteps for each task, with learning trajectories shown in Figure \ref{fig:victim_train}.

The adversarial agent, trained using the STAR framework, consists of an adversarial policy and value function. The masking function within the adversarial policy is parameterized by neural networks, as detailed in Table \ref{hyper}. The adversarial reward is defined as the negative of the victim agent’s reward to minimize artificial effort introduced by reward shaping:
\begin{equation}
R_{\text{adv}} = - R_{\text{vic}}.
\end{equation}
The adversarial agent's training spans $2\times10^3$ steps, with the trajectory shown in Figure \ref{fig:adv_train}.





\begin{table}[]
\renewcommand{\arraystretch}{1.2}
\centering
\caption{Hyperparameters}
\begin{tabular}{ll}
\hline
\multicolumn{2}{c}{Network Architecture} \\
\hline
Victim Policy Network & 2 FC layers, 128 neurons each \\
Victim Value Network & 2 FC layers, 128 neurons each \\
Adversarial Policy Network & 3 FC layers, 64 neurons each \\
Adversarial Value Network & 3 FC layers, 64 neurons each \\
\hline
\multicolumn{2}{c}{Victim Agent Hyperparameters} \\
\hline
Clipping parameter & 0.2 \\
Discount factor & 0.998 \\
GAE parameter & 0.95 \\
Value loss coefficient & 0.5 \\
Initial learning rate & $5\times10^{-4}$ \\
Maximum gradient norm & 0.5 \\
Batch size & 64 \\
adaptive learning rate KL target & 0.01 \\
\hline
\multicolumn{2}{c}{Adversarial Agent Hyperparameters} \\
\hline
Learning rate & $3\times10^{-4}$ \\
Temperature & 1.0 \\
Entropy weight coefficient & $2\times10^{-4}$ \\
Interpolation factor & $0.75$ \\
Scaling factor & $1\times10^{-4}$ \\
\hline
\end{tabular}
\label{hyper}
\end{table}

\textbf{Performance metrics.} We utilize the following three metrics to compare the performance of STAR and state-of-the-art attack methods. 

\begin{itemize}
    \item \textbf{Reward}: The average per-step reward obtained by the victim agent. In reinforcement learning, the reward reflects the feedback from the environment based on the agent's actions. A higher reward indicates that the agent is making decisions that lead to more favorable outcomes as defined by the task objectives.  

    \item \textbf{Forward Velocity}: The mean translational velocity of the robot in its forward direction (m/s). This metric measures how efficiently the robot moves forward. A higher forward velocity often indicates smoother and more effective movement.  

    \item \textbf{Fall Rate (\%)}: The probability of balance failure, calculated as the ratio between \textit{fall} incidents and total control steps. This metric evaluates the stability of the agent. A lower fall rate means the robot is better at maintaining its balance, which is crucial for reliable and safe operation.  
\end{itemize}

\begin{figure}[t]
\centering
\subfloat[]{%
    \includegraphics[width=0.5\columnwidth]{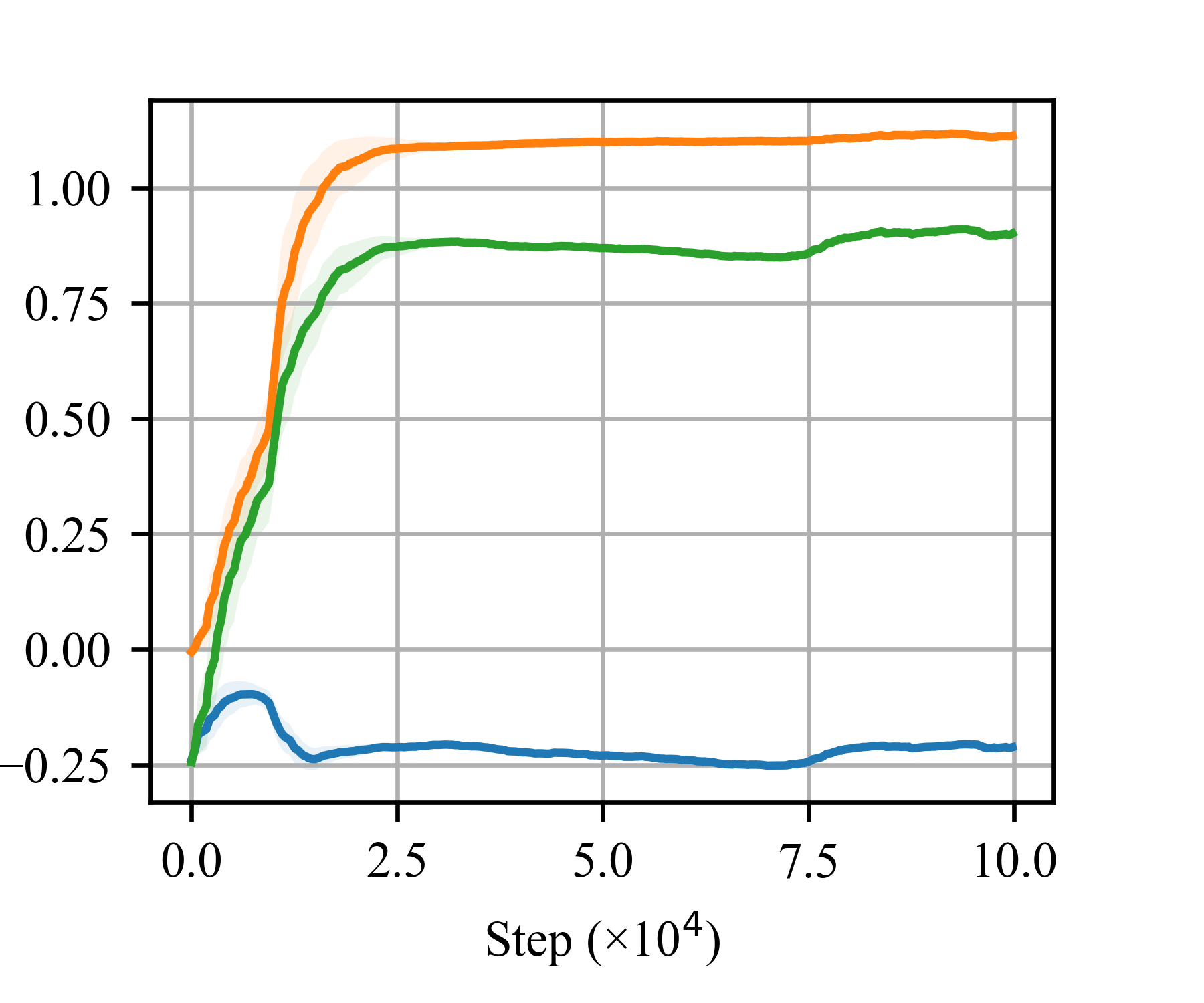}%
    \label{fig:victim_train_aliengo}}
\hfill
\subfloat[]{%
    \includegraphics[width=0.5\columnwidth]{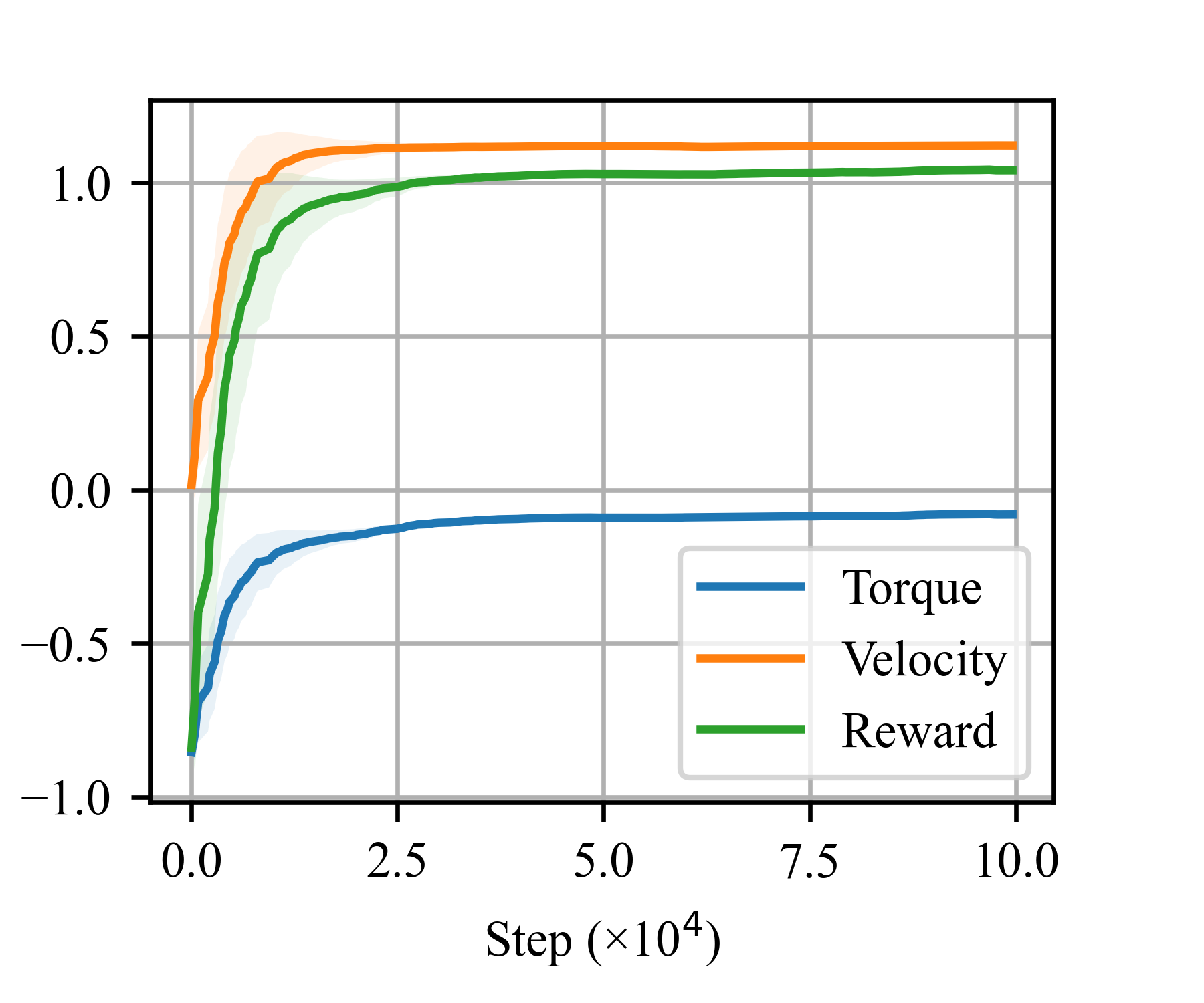}%
    \label{fig:victim_train_anymal}}
\caption{Training trajectories of torque control, velocity control, and reward of the victim agents. (a) Aliengo locomotion. (b) ANYmal locomotion.}
\label{fig:victim_train}
\end{figure}

\textbf{Baselines.}
To evaluate the performance of STAR, we compare it with the following alternatives. 

\begin{itemize}
    \item \textbf{Random Attack}: A simple baseline that applies uniform random noise as perturbations, serving as a naive benchmark for comparison.
    
    \item \textbf{FGSM} \cite{goodfellow2014explaining}: The Fast Gradient Sign Method (FGSM) is a single-step attack that perturbs the input in the direction of the gradient's sign, aiming to maximize the loss with minimal computational overhead.
    
    \item \textbf{DI\textsuperscript{2}-FGSM} \cite{xie2019improving}: An extension of FGSM that incorporates input transformations, such as resizing and padding, before computing gradients. This approach introduces input diversity, which has been shown to improve transferability to black-box models by mitigating overfitting to specific input representations.
    
    \item \textbf{MI-FGSM} \cite{dong2018boosting}: A variant of iterative FGSM that integrates a momentum term to stabilize gradient updates and reduce the likelihood of being trapped in local optima. This method is designed to enhance the transferability of adversarial examples to black-box models.
    
    \item \textbf{NI-FGSM} \cite{lin2019nesterov}: An extension of MI-FGSM that incorporates Nesterov accelerated gradient to refine the update direction by approximating future gradient information.
    
    \item \textbf{R+FGSM} \cite{tramer2017ensemble}: A variation of FGSM that introduces a small random perturbation to the input before applying the gradient-based modification. This is intended to reduce sensitivity to local gradient structures, potentially improving robustness in certain scenarios.
    
    \item \textbf{PGD} \cite{madry2017towards}: Projected Gradient Descent (PGD) is an iterative extension of FGSM that applies multiple gradient ascent steps, each followed by a projection onto the allowed perturbation space. 
    
    \item \textbf{TPGD} \cite{zhang2019theoretically}: A variant of PGD in which the cross-entropy loss is replaced with KL divergence loss for adversarial example generation. This alternative loss design is intended to improve attack performance against models trained with label smoothing or other robustness-promoting techniques.
    
    \item \textbf{EOT-PGD} \cite{liu2018adv}: An extension of PGD that incorporates the Expectation over Transformation (EOT) framework, where random transformations (e.g., rotations, translations) or model variations are applied during each attack iteration. This method aims to generate adversarial examples that maintain effectiveness across a distribution of input variations.

\end{itemize}

For a fair comparison, given a perturbation budget $\epsilon$, we perform iterative attacks using $N={10,20}$ steps with step size $\alpha=\epsilon/N$.

\begin{figure}[t]
\centering
\subfloat[]{%
    \includegraphics[width=0.5\columnwidth]{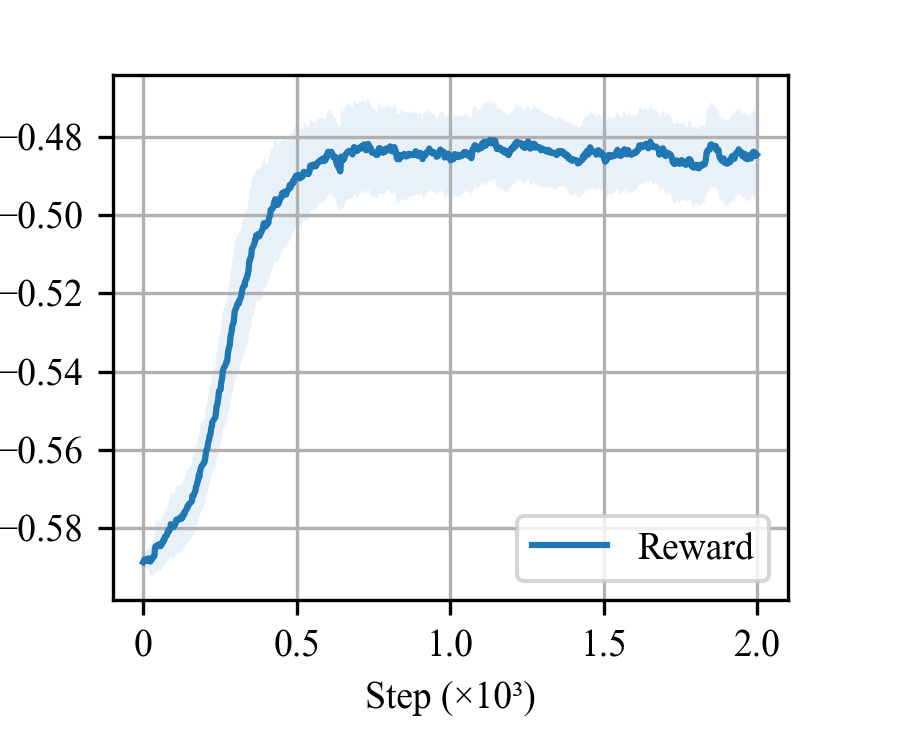}%
    \label{fig:adv_AliengoReward}}
\hfill
\subfloat[]{%
    \includegraphics[width=0.5\columnwidth]{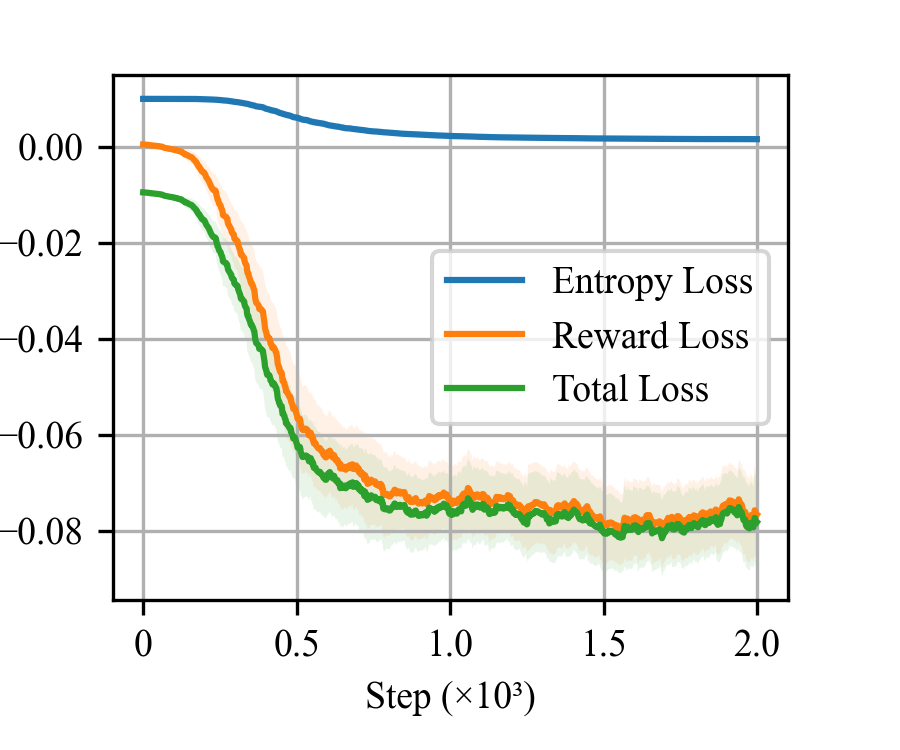}%
    \label{fig:adv_Aliengo_losses}}
\\
\subfloat[]{%
    \includegraphics[width=0.5\columnwidth]{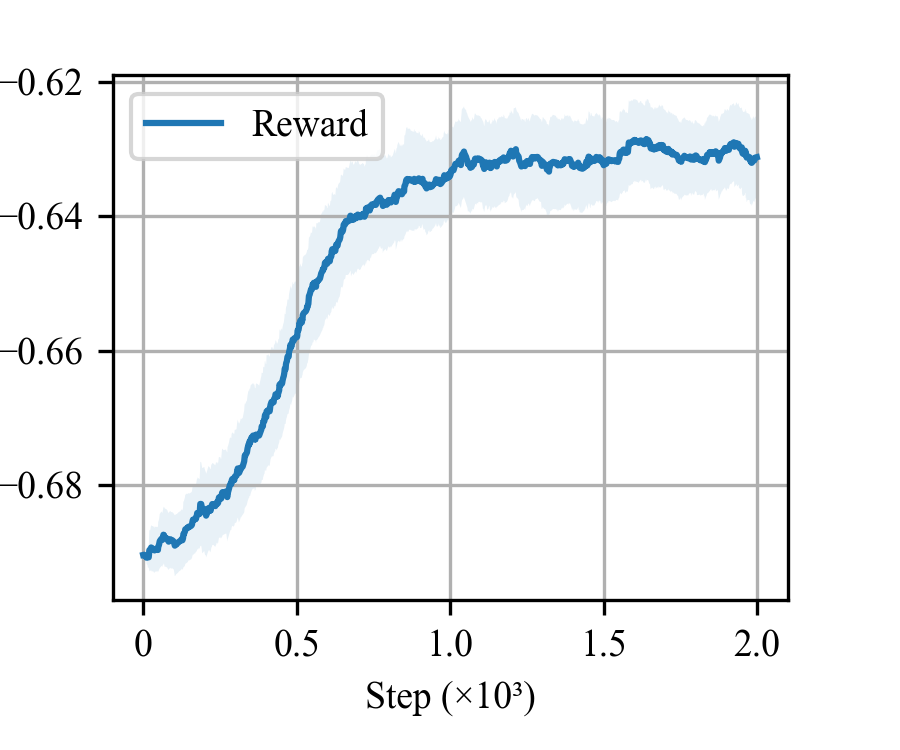}%
    \label{fig:adv_ANYmalReward.png}}
\hfill
\subfloat[]{%
    \includegraphics[width=0.5\columnwidth]{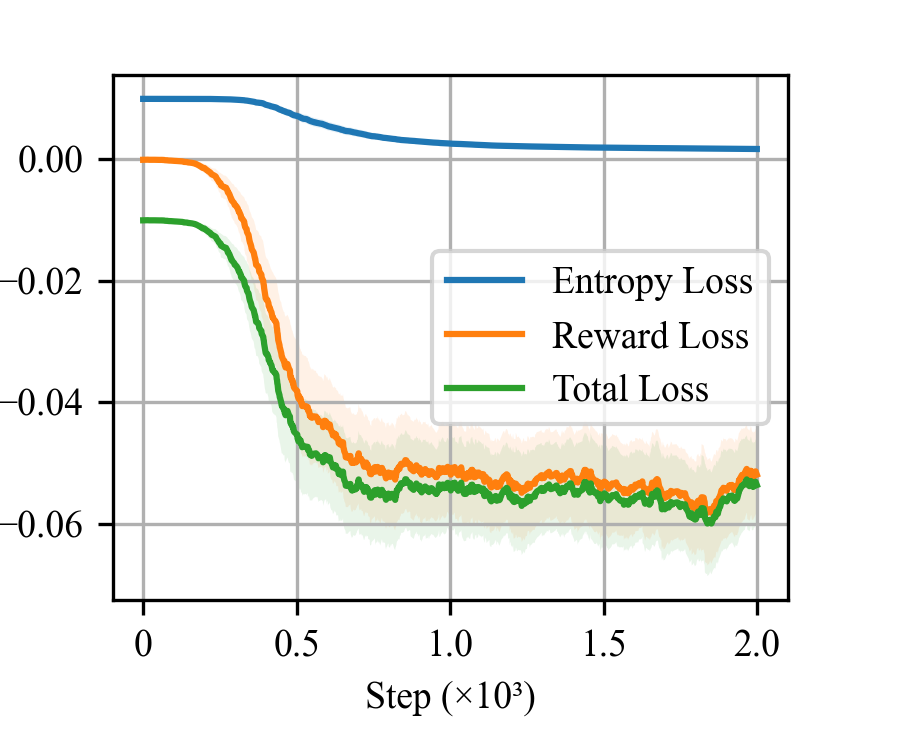}%
    \label{fig:adv_ANYmal_losses.png}}
\caption{Training trajectories of the adversarial agents. (a) Aliengo reward. (b) Aliengo loss. (c) ANYmal reward. (d) ANYmal loss.}
\label{fig:adv_train}
\end{figure}

\begin{table*}[t]

\caption{Aliengo Locomotion results, reported as mean ± SD over 20 episodes, with best in bold and suboptimal underlined. The ↑ indicates higher value is better, while the ↓ indicates lower value is better.}
\centering
\footnotesize
\begin{tabular}{|l|c|c|c|c|c|c|c|}
\hline
\multicolumn{8}{|c|}{\textbf{Aliengo Locomotion, $\epsilon = 0.025$}} \\
\hline
\multirow{2}{*}{Method} & \multirow{2}{*}{step} & \multicolumn{2}{c|}{Reward ↓} & \multicolumn{2}{c|}{Forward Velocity ↓} & \multicolumn{2}{c|}{Fall Rate (\%) ↑} \\
\cline{3-8}
& & Raw & Drop & Raw & Drop & Raw & Rise \\
\hline
No Attack & - & 0.941 ± 0.010& 0.000\%
& 3.777 ± 0.011& 0.000\%
& 0.464 ± 0.018& 0.000\%
\\
\hline
Random & - & 0.927 ± 0.015& 1.488\%
& 3.754 ± 0.029& 0.609\%
& 0.491 ± 0.041& 0.027\%
\\
\hline
FGSM & - & 0.819 ± 0.033& 12.965\%
& 3.596 ± 0.059& 4.792\%
& 0.779 ± 0.094& 0.315\%
\\
\hline
\multirow{2}{*}{\centering DI\textsuperscript{2}-FGSM} & 10 & 0.718 ± 0.050& 23.698\%
& 3.455 ± 0.092& 8.525\%
& 1.119 ± 0.122& 0.655\%
\\
& 20 & \underline{0.674 ± 0.084}& \underline{28.374\%}
& \underline{3.404 ± 0.122}& \underline{9.876\%}
& \underline{1.225 ± 0.255}& \underline{0.761\%}
\\
\hline
\multirow{2}{*}{\centering MI-FGSM} & 10 & 0.787 ± 0.053& 16.366\%
& 3.561 ± 0.088& 5.719\%
& 0.880 ± 0.130& 0.416\%
\\
& 20 & 0.789 ± 0.024& 16.153\%
& 3.557 ± 0.045& 5.825\%
& 0.830 ± 0.065& 0.366\%
\\
\hline
\multirow{2}{*}{\centering NI-FGSM} & 10 & 0.791 ± 0.043& 15.940\%
& 3.567 ± 0.064& 5.560\%
& 0.893 ± 0.150& 0.429\%
\\
& 20 & 0.795 ± 0.042& 15.515\%
& 3.577 ± 0.059& 5.295\%
& 0.846 ± 0.090& 0.382\%
\\
\hline
\multirow{2}{*}{\centering R+FGSM} & 10 & 0.709 ± 0.054& 24.655\%
& 3.430 ± 0.099& 9.187\%
& 1.116 ± 0.166& 0.652\%
\\
& 20 & 0.689 ± 0.053& 26.780\%
& 3.427 ± 0.096& 9.267\%
& \underline{1.175 ± 0.163}& \underline{0.711\%}
\\
\hline
\multirow{2}{*}{\centering PGD} & 10 & 0.734 ± 0.068& 21.998\%
& 3.463 ± 0.124& 8.313\%
& 1.151 ± 0.191& 0.687\%
\\
& 20 & 0.701 ± 0.075& 25.505\%
& 3.425 ± 0.126& 9.320\%
& 1.172 ± 0.166& 0.708\%
\\
\hline
\multirow{2}{*}{\centering TPGD} & 10 & 0.721 ± 0.063& 23.379\%
& 3.411 ± 0.151& 9.690\%
& 1.004 ± 0.158& 0.540\%
\\
& 20 & 0.734 ± 0.044& 21.998\%
& 3.497 ± 0.062& 7.413\%
& 1.062 ± 0.120& 0.598\%
\\
\hline
\multirow{2}{*}{\centering EOTPGD} & 10 & 0.706 ± 0.072& 24.973\%
& 3.423 ± 0.136& 9.373\%
& 1.036 ± 0.186& 0.572\%
\\
& 20 & \underline{0.685 ± 0.081}& \underline{27.205\%}
& \underline{3.407 ± 0.138}& \underline{9.796\%}
& 1.132 ± 0.164& 0.668\%
\\
\hline
STAR (Ours) & - & \textbf{0.622 ± 0.079}&\textbf{ 33.900\%}
& \textbf{3.268 ± 0.166}& \textbf{13.476\%}
& \textbf{1.420 ± 0.252}& \textbf{0.956\%}
\\
\hline
\end{tabular}
\label{aliengo_tab}
\end{table*}

\section{Evaluation}
\subsection{Comparative Experiments}
Tables \ref{aliengo_tab} and \ref{anymal_tab} present the comparative experimental results of robotic locomotion between the two systems, with each configuration evaluated through 20 independent trials. 
For each performance metric, we report both the raw values and the relative changes (reduction/increase) with respect to the non-attacked baseline, expressed as mean ± standard deviation. Note that the relative changes for reward and forward velocity are calculated multiplicatively, while for fall rate, being a ratio metric itself, we employ additive comparison. The optimal method is highlighted in bold, with two competitive alternatives denoted by underlining. 

{\bf{Reward.}} For reward degradation, STAR consistently outperforms all other attack methods across both Aliengo and ANYmal locomotion tasks. Specifically, for Aliengo, STAR leads to a 33.900\% reward drop, which is significantly higher than the second-best method, DI²-FGSM (28.374\%). Similarly, for ANYmal, STAR achieves a 45.377\% reduction, surpassing DI²-FGSM's 38.151\%. This demonstrates the superior adversarial effectiveness of STAR in manipulating the learned policies of robotic locomotion, leading to more severe performance degradation compared to traditional adversarial attack baselines.

\begin{table*}[t]

\caption{ANYmal Locomotion results, reported as mean ± SD over 20 episodes, with best in bold and suboptimal underlined. The ↑ indicates higher value is better, while the ↓ indicates lower value is better.}
\renewcommand{\arraystretch}{1}
\centering
\begin{tabular}{|l|c|c|c|c|c|c|c|}
\hline
\multicolumn{8}{|c|}{\textbf{ANYmal Locomotion, $\epsilon = 0.1$}} \\
\hline
\multirow{2}{*}{Method} & \multirow{2}{*}{step} & \multicolumn{2}{c|}{Reward ↓} & \multicolumn{2}{c|}{Forward Velocity ↓} & \multicolumn{2}{c|}{Fall Rate (\%) ↑} \\
\cline{3-8}
& & Raw & Drop & Raw & Drop & Raw & Rise \\
\hline
No Attack & - & 1.141 ± 0.000& 0.000\%
& 3.973 ± 0.000& 0.000\%
& 0.025 ± 0.000& 0.000\%
\\
\hline
Random & - & 1.128 ± 0.010& 1.382\%
& 3.958 ± 0.021& 0.397\%
& 0.029 ± 0.009& 0.004\%
\\
\hline
FGSM & - & 0.947 ± 0.039& 20.616\%
& 3.607 ± 0.091& 9.690\%
& 0.330 ± 0.067& 0.305\%
\\
\hline
\multirow{2}{*}{\centering DI\textsuperscript{2}-FGSM} & 10 & \underline{0.785 ± 0.044}& \underline{37.832\%}
& \underline{3.232 ± 0.112}& \underline{19.619\%}
& \underline{0.624 ± 0.088}& \underline{0.599\%}
\\
& 20 & \underline{0.782 ± 0.041}& \underline{38.151\%}
& \underline{3.233 ± 0.113}& \underline{19.592\%}
& \underline{0.620 ± 0.094}& \underline{0.595\%}
\\
\hline
\multirow{2}{*}{\centering MI-FGSM} & 10 & 0.875 ± 0.042& 28.268\%
& 3.447 ± 0.105& 13.926\%
& 0.451 ± 0.081& 0.426\%
\\
& 20 & 0.878 ± 0.042& 27.949\%
& 3.461 ± 0.103& 13.556\%
& 0.445 ± 0.082& 0.420\%
\\
\hline
\multirow{2}{*}{\centering NI-FGSM} & 10 & 0.870 ± 0.039& 28.799\%
& 3.444 ± 0.091& 14.006\%
& 0.458 ± 0.077& 0.433\%
\\
& 20 & 0.876 ± 0.027& 28.162\%
& 3.454 ± 0.061& 13.741\%
& 0.450 ± 0.048& 0.425\%
\\
\hline
\multirow{2}{*}{\centering R+FGSM} & 10 & 0.802 ± 0.043& 36.026\%
& 3.282 ± 0.113& 18.295\%
& 0.582 ± 0.087& 0.557\%
\\
& 20 & 0.794 ± 0.055& 36.876\%
& 3.262 ± 0.142& 18.824\%
& 0.594 ± 0.115& 0.569\%
\\
\hline
\multirow{2}{*}{\centering PGD} & 10 & 0.805 ± 0.041& 35.707\%
& 3.284 ± 0.111& 18.242\%
& 0.590 ± 0.091& 0.565\%
\\
& 20 & 0.804 ± 0.037& 35.813\%
& 3.281 ± 0.099& 18.321\%
& 0.595 ± 0.088& 0.570\%
\\
\hline
\multirow{2}{*}{\centering TPGD} & 10 & 0.829 ± 0.035& 33.156\%
& 3.347 ± 0.090& 16.574\%
& 0.530 ± 0.074& 0.505\%
\\
& 20 & 0.826 ± 0.040& 33.475\%
& 3.343 ± 0.105& 16.680\%
& 0.532 ± 0.084& 0.507\%
\\
\hline
\multirow{2}{*}{\centering EOTPGD} & 10 & 0.800 ± 0.041& 36.238\%
& 3.277 ± 0.108& 18.427\%
& 0.581 ± 0.087& 0.556\%
\\
& 20 & 0.791 ± 0.041& 37.194\%
& 3.257 ± 0.110& 18.957\%
& 0.589 ± 0.082& 0.564\%
\\

\hline
STAR (Ours) & - & \textbf{0.714 ± 0.061}& \textbf{45.377\%}
& \textbf{3.053 ± 0.168}& \textbf{24.358\%}
& \textbf{0.746 ± 0.129}& \textbf{0.721\%}
\\
\hline
\end{tabular}
\label{anymal_tab}
\end{table*}

{\bf{Forward Velocity.}} For forward velocity reduction, STAR demonstrates predominant efficacy, inducing the largest drop in velocity among all methods. For Aliengo, STAR reduces the velocity by 13.476\%, outperforming EOTPGD (9.796\%) and DI²-FGSM (9.876\%). Similarly, in ANYmal locomotion, STAR achieves a 24.358\% velocity drop, significantly exceeding the second-best attack, DI²-FGSM (19.619\%). Since forward velocity is a crucial indicator of the robot’s ability to maintain stable locomotion, STAR's superior performance in reducing this metric highlights its effectiveness in destabilizing robotic movement.  

{\bf{Fall Rate.}} STAR also exhibits the highest fall rate increase, making it the most effective attack in forcing robotic failures. For Aliengo, STAR increases the fall rate to 0.956\%, surpassing DI²-FGSM (0.761\%) and R+FGSM (0.711\%). In the ANYmal setup, STAR further demonstrates its effectiveness by increasing the fall rate to 0.721\%, exceeding DI²-FGSM (0.599\%) and R+FGSM (0.569\%). Since the fall rate is a direct measure of catastrophic failure in robotic locomotion, these results confirm STAR's superiority in disrupting stable movement, making it the most potent adversarial attack on both robotic platforms.

\subsection{Behavior Analysis}
\begin{figure*}
    \centering
    \includegraphics[width=1\linewidth]{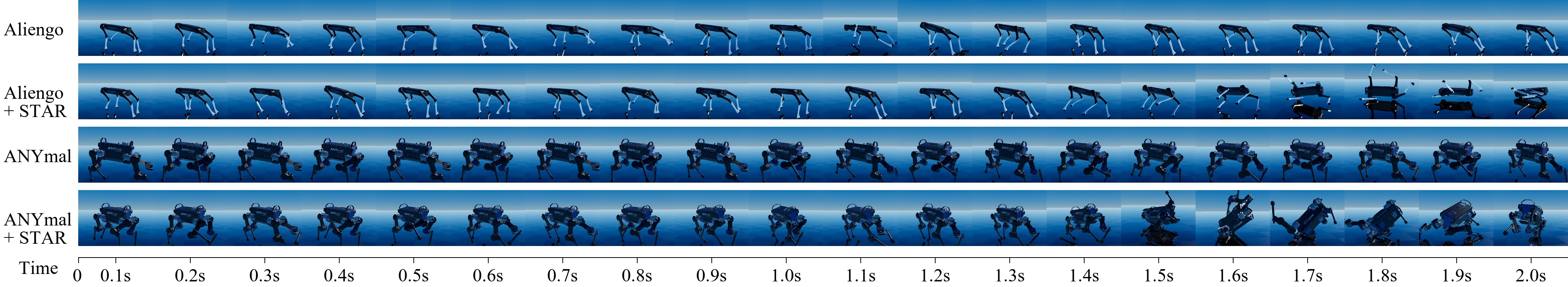}
    \caption{Comparative motion sequences of quadrupedal robots (Aliengo and ANYmal) under normal and adversarial conditions over a 2-second duration, sampled at 0.1 second intervals.}
    \label{fig:movie}
\end{figure*}
\begin{figure*}
    \centering
    \includegraphics[width=1\linewidth]{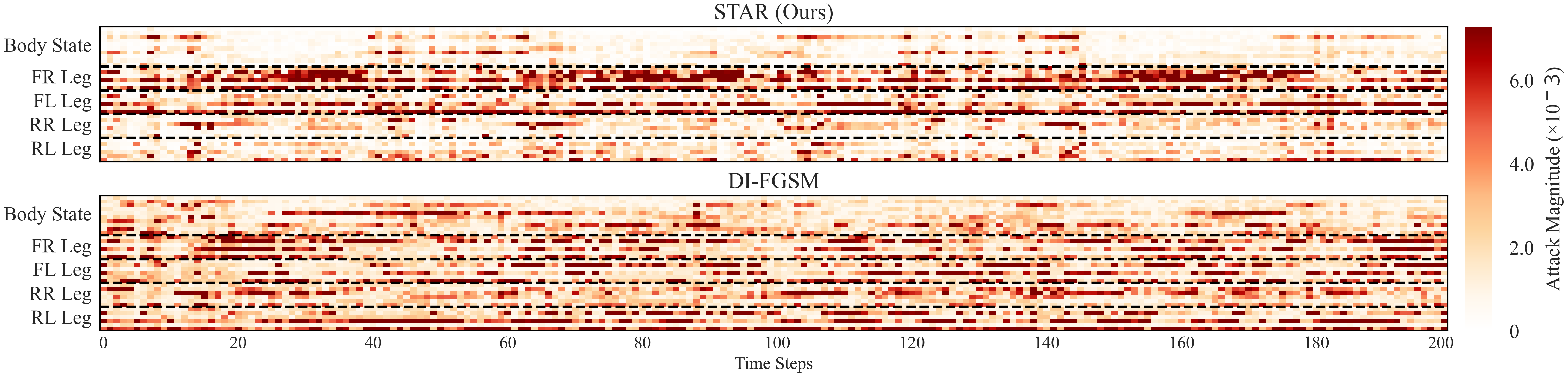}
    \caption{Temporal visualization of perturbation magnitude across body states and leg components (FR: Front Right, FL: Front Left, RR: Rear Right, RL: Rear Left) comparing STAR (Ours) and DI²-FGSM methods over 200 time steps. Darker color indicates higher perturbation magnitude.}
    \label{fig:heatmap}
\end{figure*}
Figure \ref{fig:movie} compares the locomotion of two quadrupedal robots under normal conditions and STAR-induced adversarial perturbations, with motion sequences sampled at 0.1-second intervals over 2 seconds to capture gait stability and disruption. Aliengo (first row) maintains a stable gait without interference, while STAR perturbations (second row) cause deviations at around 1.5 seconds, leading to instability and balance loss by the end of the sequence. Similarly, ANYmal (third row) moves smoothly under normal conditions, whereas STAR (fourth row) triggers earlier destabilization at around 1.2 seconds and collapses by 1.5 seconds. This visual evidence intuitively demonstrates the effectiveness of STAR in inducing instability and disrupting gait coordination.

Figure \ref{fig:heatmap} visualizes the perturbation distribution across observation dimensions using temporal heatmaps. We compare STAR with the optimal baseline DI\textsuperscript{2}-FGSM. The observation space is divided into five groups, separated by dashed lines in the figure: (i) Body State, including body height, orientation, linear velocities  and angular velocities; and four leg groups corresponding to (ii) Front-Right (FR), (iii) Front-Left (FL), (iv) Rear-Right (RR), and (v) Rear-Left (RL) legs, where each leg contains hip, thigh, and calf joint angles and velocities. It can be observed that STAR shows clear perturbation patterns, mainly targeting body orientation and front leg states. The focus on front leg configurations and body states is significant since they directly control stability and balance during forward movement. Notably, these perturbation patterns demonstrate temporal periodicity, which may correlate with the cyclic nature of individual steps in quadrupedal locomotion. This shows STAR can identify both key state variables and important time windows during locomotion. In contrast, DI\textsuperscript{2}-FGSM applies a more uniform perturbation distribution without selectively targeting key state components. Its perturbation budget is spread across less impactful dimensions.

\begin{figure}[t]
\centering
\subfloat[]{%
    \includegraphics[width=0.5\columnwidth]{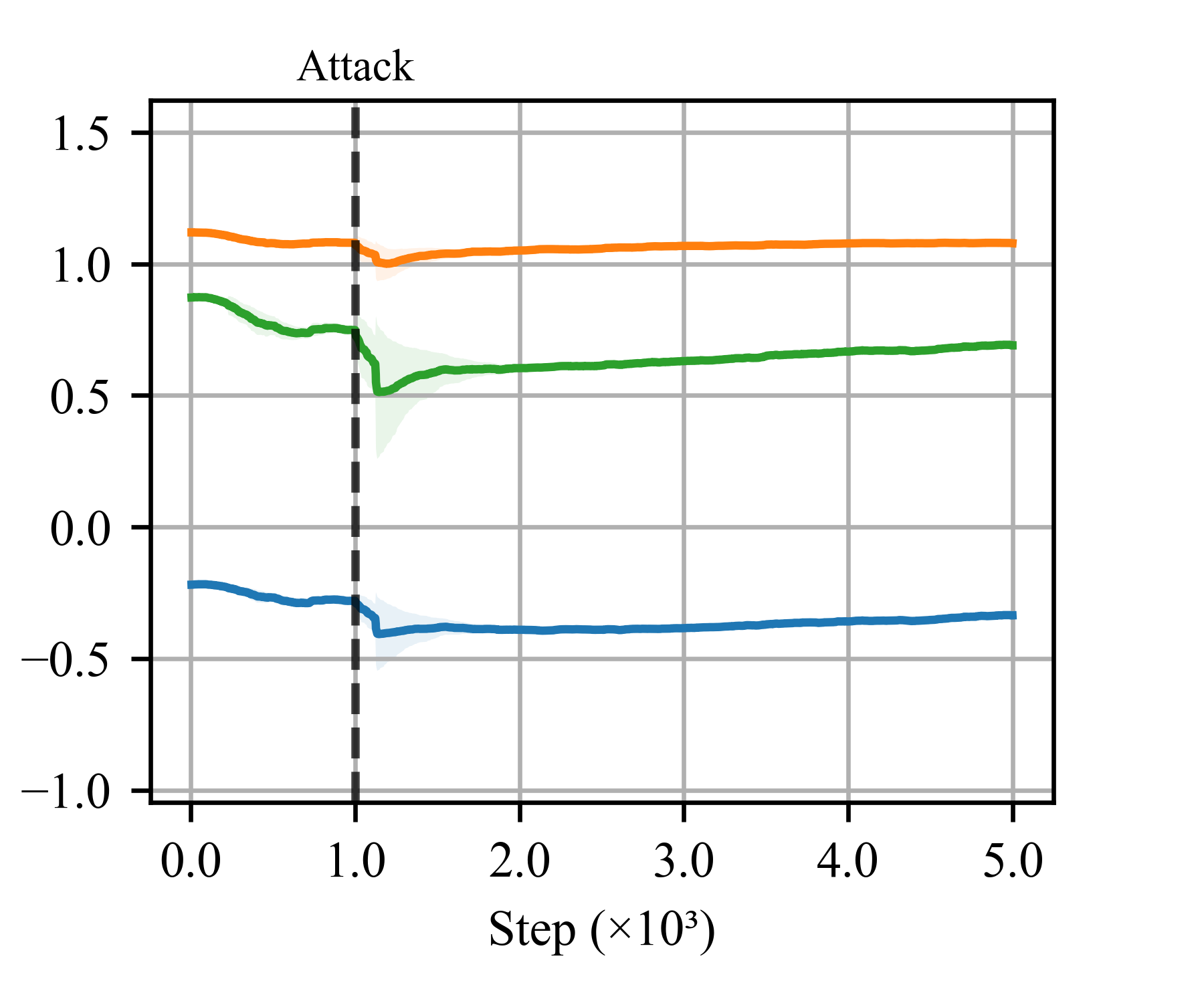}%
    \label{fig:def_train_aliengo}}
\hfil
\subfloat[]{%
    \includegraphics[width=0.5\columnwidth]{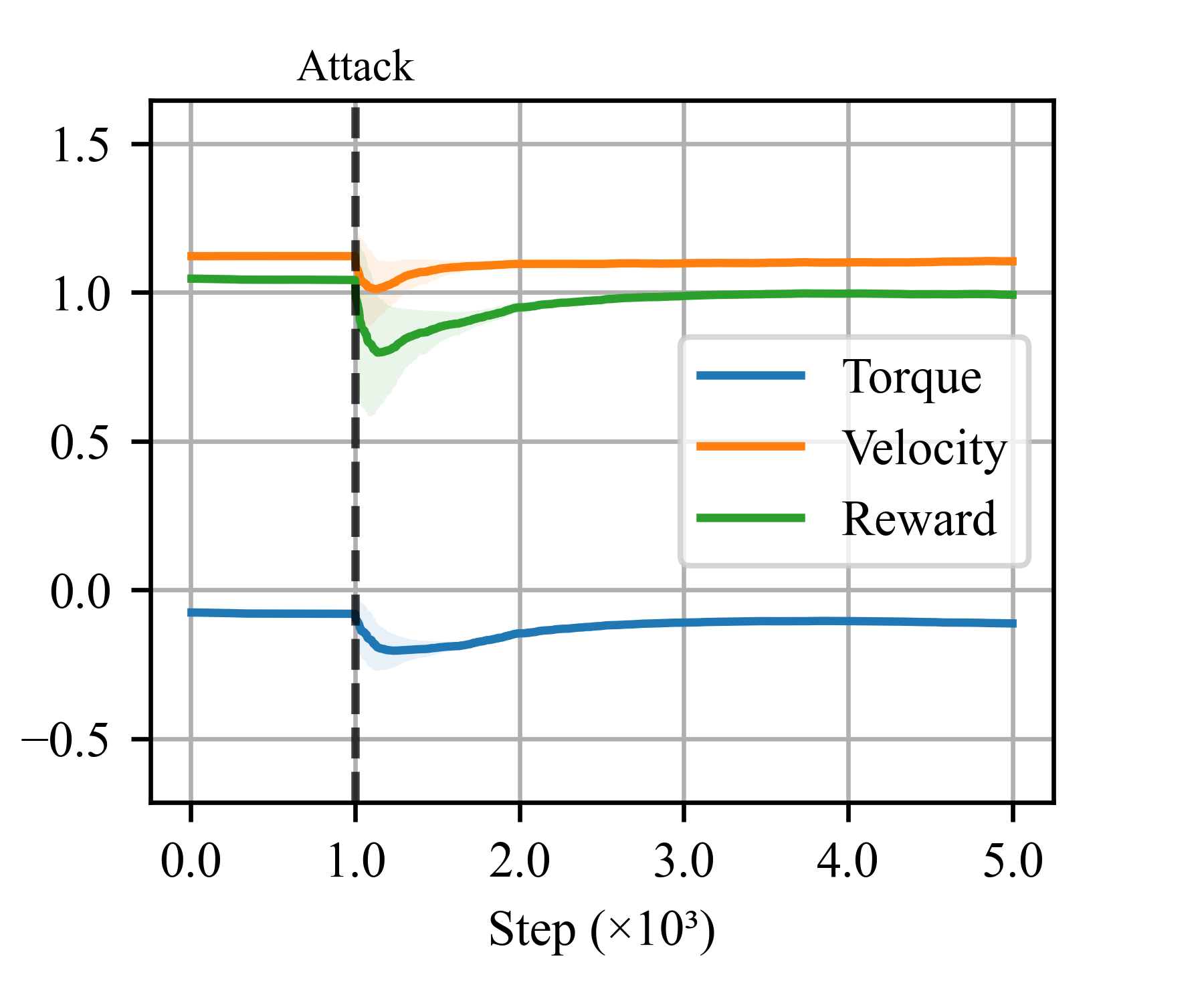}%
    \label{fig:def_train_anymal}}
\caption{Training trajectory during adversarial defense. The performance initially drops due to adversarial attacks but gradually recovers as the victim agent develops robustness. (a) Aliengo locomotion. (b) ANYmal locomotion.}
\label{fig:def_train}
\end{figure}

\begin{figure}[t]
\centering
\includegraphics[width=1\columnwidth]{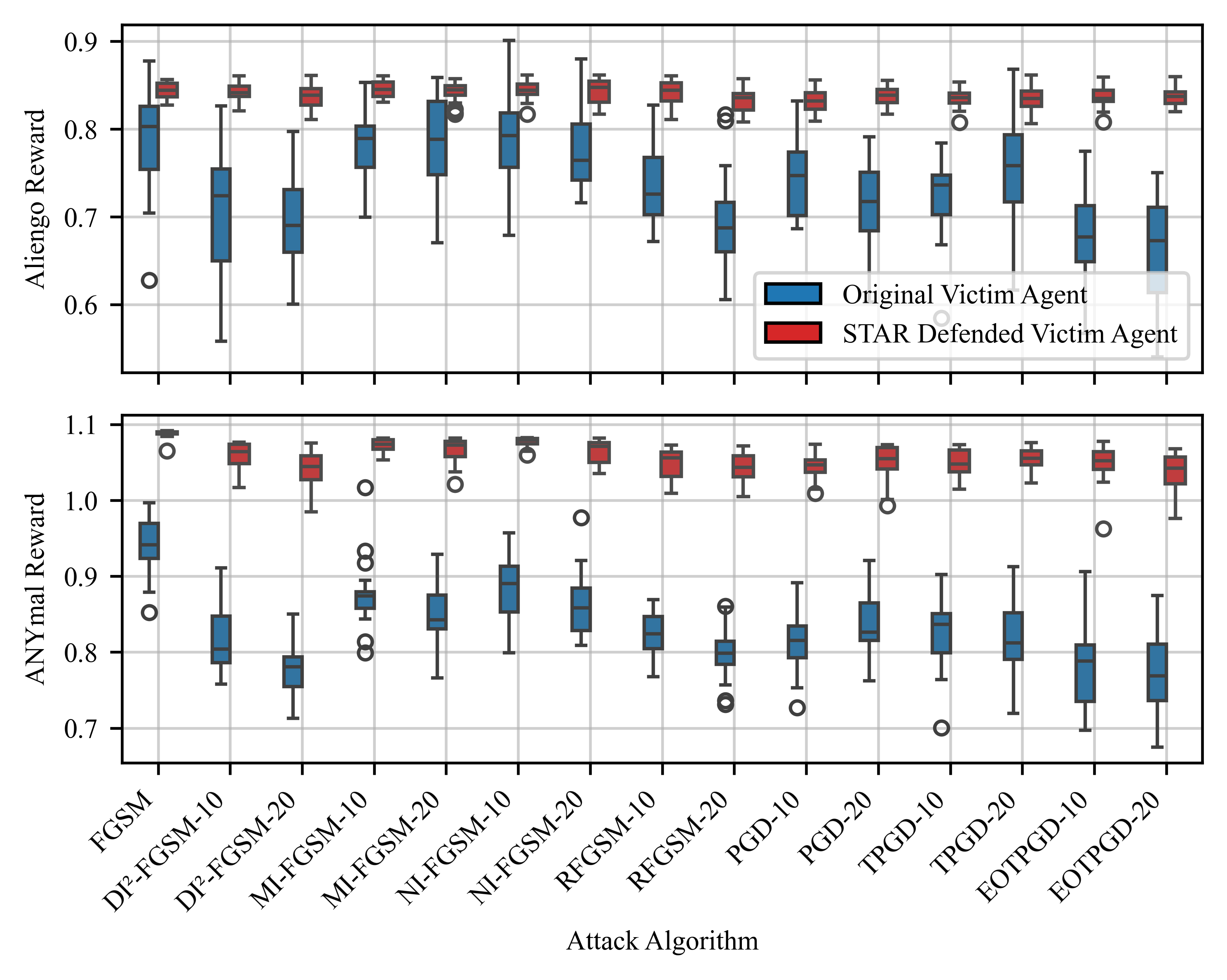}
\caption{Performance comparison between original and STAR defended victim agents under different adversarial attacks. For iterative attack methods, the notation algorithm-N denotes $N$ iteration steps}
\label{fig:defense_comparison}
\end{figure}

\subsection{Post-Defense Robustness}
Figure \ref{fig:def_train} shows the victim agent's training trajectory under STAR adversarial defense, highlighting STAR's potential for enhancing model robustness. The training runs for $4\times10^3$ steps at a learning rate of $3\times10^{-4}$, with the victim agent operating under STAR attacks. The adversarial perturbations are constrained within $\epsilon = 0.025$ for ANYmal and $\varepsilon = 0.1$ for Aliengo. When the STAR attack is introduced at $1\times10^3$ steps, the reward initially drops but recovers as the victim agent adapts to the adversarial attack through policy optimization.

Figure \ref{fig:defense_comparison} presents box plots comparing the robustness of both the original and defended victim agents against various attack methods, with each box representing the distribution of rewards across 20 evaluation episodes. For Aliengo, the defended victim agent demonstrates consistently strong performance with median rewards above 0.83 across all attack methods, exhibiting good stability with tight interquartile ranges. In contrast, the original victim agent shows its vulnerability, with median rewards varying between 0.65-0.8 and large variances, particularly under DI²-FGSM and EOTPGD. Similar enhancement in robustness is observed for ANYmal, where the defended victimagent has median rewards consistently above 1.0 across different attacks. The original victim agent, except for the single-step FGSM attack, suffers a significant reward drop (below 0.95). These results demonstrate that STAR effectively generates adversarial samples for robust defense development. Through adversarial training with STAR-crafted perturbations, the defended victim agent develops strong resistance against diverse attack scenarios, achieves near-nominal performance even under strong iterative attacks (e.g., DI²-FGSM-20).

\begin{figure}
    \centering
    \includegraphics[width=1\linewidth]{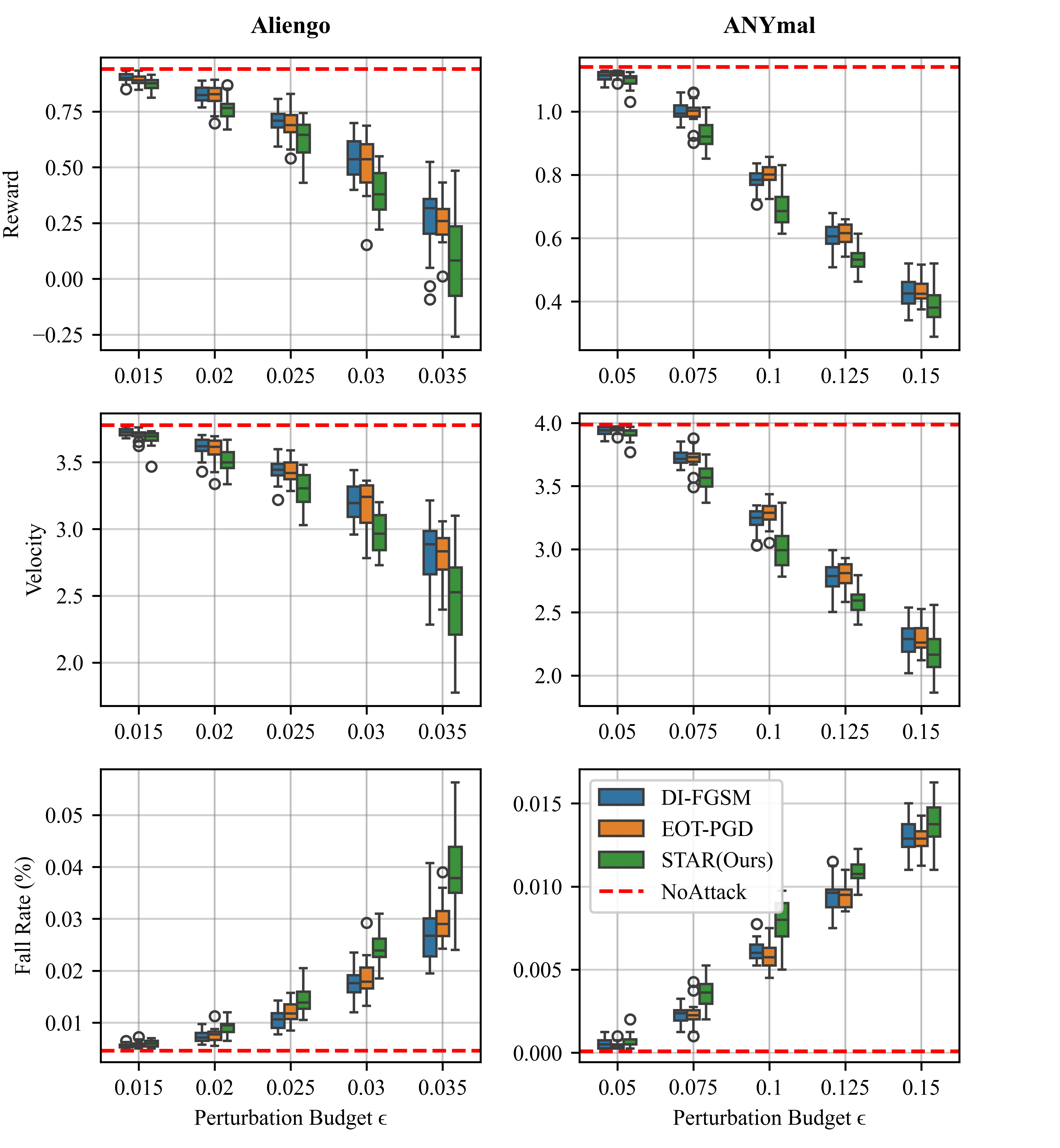}
    \caption{Performance comparison of different attack methods under varying perturbation budgets $\epsilon$.}
    \label{fig:ablation}
\end{figure}

\subsection{Ablation Study}
Figure~\ref{fig:ablation} shows the impact of perturbation budget $\epsilon$ on the attack performance. The perturbation budget $\epsilon$ serves as a crucial hyperparameter that controls the magnitude of adversarial perturbations. A larger $\epsilon$ allows for more substantial modifications to the input observations, potentially leading to more effective attacks, while a smaller $\epsilon$ ensures better imperceptibility. We compare STAR against two suboptimal baselines DI²-FGSM and EOTPGD under different $\epsilon$ values across both tasks. As the perturbation budget $\epsilon$ increases, all attack methods show enhanced effectiveness, indicated by decreasing rewards and velocities, along with rising failure rates. Our proposed STAR method consistently outperforms baseline approaches under various perturbation constraints, particularly in higher $\epsilon$ regions, demonstrating advantages over other baseline methods in terms of reward reduction, velocity decrease, and increased fall rates. This is particularly pronounced in the Aliengo locomotion task under higher perturbation budgets (when $\epsilon=0.035$). With increased perturbation allowance, our mask-based mechanism more effectively targets identified vulnerable states (e.g., critical leg configurations), resulting in improved attack performance.

\section{Conclusion}
In this paper, we address the challenge of state-aware adversarial attacks in DRL for robotic control systems. Through the proposed AVD-MDP framework, we provide theoretical foundations for analyzing adversarial attacks while considering temporal dynamics and long-term rewards. The STAR algorithm implements these insights by combining soft mask-based state targeting with information-theoretic optimization, enabling efficient generation of state-selective perturbations. Experimental results on robotic locomotion tasks demonstrate that our approach achieves improved attack effectiveness under fixed perturbation budgets compared to conventional gradient-based methods, while the resulting adversarial training enhances the robustness of DRL policies in robotic control applications.  As a potential future direction, we are looking forward to extending our method to improve the performance of various applications such as large language models~\cite{hu2024agentscomerge,lin2024splitlora,hu2021lora,fang2024automated,zhou2025survey,wang2025contemporary}, multi-modal training~\cite{fang2024ic3m,tang2024merit}, and distributed machine learning~\cite{lin2024hierarchical,lyu2023optimal,zhang2024fedac,lin2024efficient,hu2024accelerating,lin2025leo,zhang2024satfed,lin2024fedsn}.

\bibliographystyle{IEEEtran}
\bibliography{reference}

\vfill

\end{document}